\def\year{2021}\relax
\documentclass[letterpaper]{article} 
\usepackage{times}  
\usepackage{helvet} 
\usepackage{courier}  
\usepackage[hyphens]{url}  
\usepackage{graphicx} 
\usepackage{natbib}  
\usepackage{caption} 
 \usepackage{geometry}
\usepackage{makecell}
\usepackage{multicol}
\usepackage{multirow}
\usepackage{amsmath, amssymb, amsthm, amsfonts}
\usepackage{graphicx}
\usepackage{algorithm}
\usepackage{algorithmic}
\usepackage{bbm}
\usepackage{verbatim}
\usepackage{xcolor}
\usepackage{mathtools}

\usepackage{bbm}
\usepackage{booktabs}
\usepackage{arydshln}

\usepackage[inline]{enumitem}
\setlist[itemize]{leftmargin=*}
\setlist[enumerate]{leftmargin=*}

\theoremstyle{plain}

\newtheorem{proposition}{Proposition}

\newtheorem{example}{Example}
\theoremstyle{definition}
\newtheorem{definition}{Definition}

\usepackage{thmtools}
\usepackage{thm-restate}
\usepackage[capitalise]{cleveref}

\usepackage[english]{babel}
\usepackage{blindtext}

\newcommand{\argmax}{\operatorname{argmax}}
\newcommand{\argmin}{\operatorname{argmin}}
\newcommand{\dvc}{d_\mathrm{VC}}

\usepackage{xcolor}
\definecolor{anil}{rgb}{0.8, 0.0, 0.5}

\title{Classification with Strategically Withheld Data}

\usepackage{authblk}

\author[1]{Anilesh K. Krishnaswamy}
\author[2]{Haoming Li}
\author[1]{David Rein}
\author[1]{Hanrui Zhang}
\author[1]{Vincent Conitzer}
\affil[1]{Duke University}
\affil[2]{University of Southern California}

\begin{document}

\maketitle

\begin{abstract}

Machine learning techniques can be useful in applications such as credit approval and college admission. However, to be classified more favorably in such contexts, an agent may decide to strategically withhold some of her features, such as bad test scores.  This is a missing data problem with a twist: which data is missing {\em depends on the chosen classifier}, because the specific classifier is what may create the incentive to withhold certain feature values. We address the problem of training classifiers that are robust to this behavior.

We design three classification methods: {\sc Mincut}, {\sc Hill-Climbing} ({\sc HC}) and Incentive-Compatible Logistic Regression ({\sc IC-LR}). We show that {\sc Mincut} is optimal when the true distribution of data is fully known. However, it can produce complex decision boundaries, and hence be prone to overfitting in some cases. Based on a characterization of truthful classifiers (i.e., those that give no incentive to strategically hide features), we devise a simpler alternative called {\sc HC} which consists of a hierarchical ensemble of out-of-the-box classifiers, trained using a specialized hill-climbing procedure which we show to be convergent. For several reasons, {\sc Mincut} and {\sc HC} are not effective in utilizing a large number of complementarily informative features. To this end, we present {\sc IC-LR}, a modification of Logistic Regression that removes the incentive to strategically drop features. We also show that our algorithms perform well in experiments on real-world data sets, and present insights into their relative performance in different settings.

\end{abstract}

\section{Introduction}
Applicants to most colleges in the US are required to submit their scores for at least one of the SAT and the ACT. Both tests are more or less equally popular, with close to two million taking each in 2016 \citep{adams2017race}. Applicants usually take one of these two tests -- whichever works to their advantage.\footnote{\scriptsize\url{https://www.princetonreview.com/college/sat-act}} However, given the growing competitiveness of college admissions, many applicants now take both tests and then strategically decide whether to drop one of the scores (if they think it will hurt their application) or report both.\footnote{\scriptsize\url{https://blog.collegevine.com/should-you-submit-your-sat-act-scores/}} The key issue here is that
it is impossible to distinguish between an applicant who takes both tests but reports only one, and an applicant that takes only one test---for example because the applicant simply took the one required by her school,
the dates for the other test did not work with her schedule, or for other reasons that are not strategic in nature.\footnote{\scriptsize\url{https://blog.prepscholar.com/do-you-need-to-take-both-the-act-and-sat}}

Say a college wants to take a principled machine learning approach to making admission decisions based on the scores from these two tests.  For simplicity, assume no other information is available. Assume that the college has enough historical examples that contain the scores of individuals (on whichever tests are taken, truthfully reported) along with the corresponding ideal (binary) admission decisions.\footnote{\scriptsize We make these assumptions more generally throughout the paper.} Based on this data, the college has to choose a decision function that determines which future applicants are accepted. If this function is known to the applicants, they are bound to strategize and use their knowledge of the decision function to decide the scores they report.\footnotemark[\value{footnote}] 
How can the classifier be trained to handle strategic reporting of scores at prediction time?

To see the intricacies of this problem, let us consider a simple example.
\begin{example}
Say the scores for each of the two tests (SAT and ACT) take one of two values: $h$ (for high) or $l$ (for low). Let $*$ denote a missing value. Then there are eight possible inputs (excluding $(*,*)$ since at least one score is required): $(h,h)$, $(h,l)$, $(l,h)$, $(l,l)$, $(h,*)$, $(*,h)$, $(l,*)$ and $(*,l)$. Assume the natural distribution (without any withholding) over these inputs is known, and so are the conditional probabilities of the label $Y \in \{0,1\}$, as shown below:

\begin{table}[ht]
\centering
\small
\caption{True distribution of inputs and targets:} 
\setlength\tabcolsep{1.75pt}
\begin{tabular}{c|cccccccc}
 $X$ & $(h,h)$ & $(h,l)$& $(l,h)$& $(l,l)$& $(h,*)$& $(*,h)$& $(l,*)$ & $(*,l)$\\ 
 $Pr(X)$ & $1/8$&  $1/8$&  $1/8$&  $1/8$&  $1/8$&  $1/8$&  $1/8$&  $1/8$\\\hline
 $Pr(Y=1 \mid X)$ & $0.9$&  $0.7$&  $0.3$&  $0.1$&  $0.6$&  $0.6$&  $0.2$&  $0.2$\\
 $Pr(Y=0 \mid X)$ & $0.1$&  $0.3$&  $0.7$&  $0.9$&  $0.4$&  $0.4$&  $0.8$&  $0.8$
 \smallskip
\end{tabular}
\label{tab:example1}
\end{table}

Assume $Y=1$ is the more desirable "accept" decision. Then, ideally, we would like to predict $\widehat{Y}=1$ whenever $X \in \{(h,h),(h,l),(h,*),(*,h)\}$. However, the strategic reporting of scores at prediction time effectively means, for example, that an input $(*,h)$ cannot be assigned the accept decision of $\widehat{Y}=1$ unless the same is done for $(l,h)$ as well; otherwise, someone with $(l,h)$ would simply not report the first test, thereby misreporting $(*,h)$ and being accepted. Taking this into account, 
the classifier with minimum error is given by $\widehat{Y}=1$ whenever $X \in \{(h,h),(h,l),(h,*)\}$.
\end{example}

There are many other settings where a similar problem arises. Many law schools now allow applicants to choose between the GRE and the traditional LSAT.\footnote{\scriptsize \url{https://www.ets.org/gre/revised_general/about/law/}}  Recently, as a result of the COVID-19 pandemic, universities have implemented optional pass/fail policies, where students can choose to take some or all of their courses for pass/fail credit, as opposed to a standard letter grade that influences their GPA.  They are often able to decide the status after already knowing their performance in the course. For credit scoring, some individuals might not report some of their information, especially if it is not mandatory by law \citep{florezlopez2010effects}.

The ability of strategic agents to withhold some of their features at prediction time poses a challenge only when the data used to train the classifier has some naturally missing components to begin with. For if not, the {\em principal} -- e.g., the entity deciding on admissions -- can reject all agents that withhold any of their features, thereby forcing them to reveal all features. 
We focus on how a principal can best train classifiers that are robust even when there is strategic withholding of data by agents.  Our methods produce classifiers that eliminate the incentive for agents to withhold data.

\paragraph{Our contributions}
We now describe the key questions we are facing, and how we answer them. Our model is described formally in \emph{Section \ref{sec:prelim}}. All proofs are in the Supplement.

\emph{If the true input distribution is known, can we compute the optimal classifier? (Section \ref{sec:mincut}}) We answer this question in the affirmative by showing that the problem of computing the optimal classifier (Theorem \ref{thm:mincut}) in this setting reduces to the classical Min-cut problem \citep{cormen2009introduction}. This analysis gives us the {\sc Mincut} classifier, which can be computed on the empirical distribution, estimated using whatever data is available. However, since it can potentially give complex decision boundaries, it might not generalize well. 

\emph{Are there simpler classifiers that are robust to strategic withholding of features? (Section \ref{sec:greedy})} We first characterize the structure of classifiers that are ``truthful", i.e., give no incentive to strategically hide features at prediction time (Theorem \ref{thm:characterization}). Using this characterization, 
we devise a hill-climbing procedure ({\sc HC}) to train 
a hierarchical ensemble of out-of-the-box classifiers
and show that the procedure converges (Theorem \ref{thm:convergence}) as long as we have black-box access to an agnostic learning oracle. We also analytically bound the generalization error of {\sc HC} (Theorem \ref{thm:generalization}). The ensemble of {\sc HC} can be populated with any of the 
commonly used classifiers such as logistic regression, {\sc ANN}s, etc. 

Another truthful classifier we present is a modification of Logistic Regression. This method, called {\sc IC-LR} (Incentive Compatible Logistic Regression), works by encoding all features with positive values, and using positive regression coefficients -- whereby it is in every agent's best interest to report all features truthfully. {\sc IC-LR} uses Projected Gradient Descent for its training. The advantage of this method is that it can be directly to a large number of features.

\emph{How do our methods perform on real data sets? (Section \ref{sec:experiments})} We conduct experiments on several real-world data sets to test the performance of our methods, comparing them to each other, as well as to other methods that handle missing data but ignore the strategic aspect of the problem. We see that our methods perform well overall, and uncover some interesting insights on their relative performance:
\begin{enumerate}
    \item When the number of features is small, {\sc HC} is the most reliable across the board.
    \item When the number of features is small, and many of them are discrete/categorical (or suitably discretized), {\sc Mincut} and {\sc IC-LR} perform better.
    \item If a large number of features must be used, {\sc IC-LR} gives the best performance, although {\sc HC} performs reasonably well with some simple feature selection techniques.
\end{enumerate}

\paragraph{Related work}
Our work falls broadly in the area of {\em strategic machine learning}, wherein a common assumption is that strategic agents can modify their features (i.e., misreport) in certain ways (normally at some cost), either to improve outcomes based on the classifier chosen by the principal \citep{hardt2016strategic} or to influence which classifier is chosen in the first place \citep{dekel2010incentive}. The main challenge in strategic machine learning, as in this paper, is the potential misalignment between the interests of the agents and the principal. Existing results in this line of work \citep{chen2018strategyproof,kleinberg2019classifiers,haghtalab2020maximizing}, often mainly theoretical, consider classifiers of a specific form, say linear, and ways of misreporting or modifying features in that context. Our results are different in that we focus on a specific type of strategic misreporting, i.e., withholding parts of the data, and devise general methods that are robust to this behavior that, in addition to having theoretical guarantees, can be tested practically. Some experimental results \citep{hardt2016strategic} do exist -- but our work is quite different; for instance, we do not need to invent a cost function (as in \citet{hardt2016strategic}). Another major difference is that we consider generalization in the presence of strategic behavior, while most previous work does not (except for a concurrent paper \citep{zhang2021incentive}), which studies the sample complexity of PAC learning in the presence of strategic behavior).

Our problem can also be viewed as an instance of {\em automated mechanism design with partial verification} \citep{green1986partially,yu2011mechanism,kephart2015complexity,kephart2016revelation} where it is typically assumed that the feature space (usually called type space in mechanism design) is discrete and has reasonably small cardinality, and a prior distribution is known over the feature space. In contrast, the feature spaces considered in this paper consist of all possible combinations of potentially continuous feature values. Moreover, the population distribution can only be accessed by observing examples. Thus, common methodologies in automated mechanism design do not suffice for our setting.

A set of closely related (in particular, to Theorem \ref{thm:mincut}) theoretical results are those of \citet{zhang2019samples,zhang2019distinguishing,zhang2021classification} on the problem of distinguishing ``good'' agents from ``bad'' 
(where each produces a different distribution over a sample space, and the agent can misreport the set of $n$ samples that she has drawn). However, our work is different in that we consider the standard classification problem, we focus more on practical aspects, and we do not rely on the full knowledge of the input distribution.

Our work also finds a happy intersection between strategic machine learning and the literature on classification with missing data \citep{marlin2008missing}. The problem we study is also connected to \textit{adversarial classification} \citep{dalvi2004adversarial,dekel2010learning}. We discuss these connections in more detail in the Supplement.

\section{Preliminaries}
\label{sec:prelim}

We now describe our model and the requisite notation.
\paragraph{Model with strategically withheld features:} We have an input space $\mathcal{X}$, a label space $\mathcal{Y} = \{0, 1\}$, and a distribution $\mathcal{D}$ over $\mathcal{X} \times \mathcal{Y}$ which models the population. A classifier $f: \mathcal{X} \to \mathcal{Y}$ maps a combination of features to a label. Let $F = [k] = \{1, \dots, k\}$ be the set of features, each of which a data point may or may not have.
For $x \in \mathcal{X}$, let $x_i$ denote the value of its $i$-th feature ($x_i = *$ if $x$ does not have feature $i \in [k]$). For any $S \subseteq [k]$, define $x|_S$ to be the projection of $x$ onto $S$ (i.e., retain features in $S$ and drop those not in $S$): 
\begin{align*}
    (x|_S)_i =
    \begin{cases}
        x_i, & \text{if } i \in S \\
        *, & \text{otherwise}.
    \end{cases}
\end{align*}

We assume that data can be strategically manipulated at prediction (test) time in the following way: an agent whose true data point is $x$ can report any other data point $x'$ such that $x|_S = x'$ for some $S \subseteq [k]$. We use $\to$ to denote the relation between any such pair $x, x'$ ($x \to x' \iff \exists S \subseteq [k] : x|_S = x'$). Note that $\to$ is transitive, i.e., for any $x_1, x_2, x_3 \in \mathcal{X}$, $x_1 \to x_2 \text{ and } x_2 \to x_3 \Longrightarrow x_1 \to x_3$. 

We assume agents prefer label $1$ to $0$: in response to a classifier $f$, an agent with data point $x$ will always withhold
\footnote{\scriptsize In practice, $f$ might not be perfectly known, and agents might not be able to best respond. This problem does not arise for our methods, since they are truthul. For other classifiers, their accuracy may go up or down if agents fail to best-respond; but the assumption that agents best-respond is common in many such contexts.} 
features to receive label $1$ if possible, i.e., the agent will report $x' \in \argmax_{x'': x \to x''} f(x'')$. Incorporating such strategic behavior into the loss of a classifier $f$, we get
\begin{align*}
        \ell_\mathcal{D}(f) = \Pr_{(x, y) \sim \mathcal{D}}\left[y \ne \max_{x': x \to x'} f(x') \right].
\end{align*}

\paragraph{Truthful classifiers}
We will also be interested in {\em truthful} classifiers, which provably eliminate incentives for such strategic manipulation. A classifier $f$ is \emph{truthful} if for any $x, x' \in \mathcal{X}$ where $x \to x'$, $f(x) \ge f(x')$.
In other words, not withholding any features is always an optimal way to respond to a truthful classifier.
As a result, the loss of any truthful classifier $f$ in the presence of strategically withheld features has the standard form:
$\ell_\mathcal{D}(f) = \Pr_{(x, y) \sim \mathcal{D}}[f(x) \ne y]$.

Note that the so-called Revelation Principle -- which states that in the presence of strategic behavior, any classifier $f$ is equivalent to a truthful classifier $f'$ -- holds in this case because the reporting structure is transitive.\footnote{\scriptsize More details, including a formal proof, are in the Supplement.}
In other words, we are guaranteed that, for any classifier $f$, there exists a truthful classifier $f'$, such that for any $x \in \mathcal{X}$, $\max_{x': x \to x'} f(x') = f'(x)$. Therefore, we focus on truthful classifiers in our model, without loss of generality.
\section{The {\sc Mincut} Classifier} \label{sec:mincut}
We first present a method for computing an optimal classifier {\em when the input distribution is fully known}.\footnote{\scriptsize
A theoretical companion paper \cite{zhang2021automated} contains a more general version of the mincut-based algorithm.
There, the goal is to compute an optimal classifier with possibly more than 2 outcomes given perfect knowledge of the entire population distribution.
In this paper, we investigate the special case with only 2 outcomes (i.e., accept and reject), but do not assume prior knowledge about the population distribution.
} Assuming $\mathcal{X}$ is finite, our goal is to characterize a classifier $f^*$ which minimizes the loss $\ell_\mathcal{D}(.)$, for a known input distribution $\mathcal{D}$. As shorthand, define, for all $x \in \mathcal{X}$,
\begin{definition}\label{def:Dplusminus}
    $\mathcal{D}^+(x) \coloneqq \Pr_{(x', y') \sim \mathcal{D}}[x' = x \wedge y' = 1], \\ \quad \mathcal{D}^-(x) \coloneqq \Pr_{(x', y') \sim \mathcal{D}}[x' = x \wedge y' = 0]$.
\end{definition}

The basic idea here is simple: to partition $\mathcal{X}$ into two sides, one labeled $1$ and the other $0$, where the error accrued for each $x \in \mathcal{X}$ is given by $\mathcal{D}^-(x)$ or $\mathcal{D}^+(x)$, according as $x$ is labeled $1$ or $0$. Such a partition should crucially respect the constraints 
imposed by the strategic behavior of agents : if $x \to x'$, then either $x$ is labeled $1$ or $x'$ is labeled $0$.

\begin{definition}\label{def:G_dx}
 Given $\mathcal{X}$ and $\mathcal{D}$, let $G(\mathcal{D}, \mathcal{X})$ be a directed capacitated graph with vertices $V = \mathcal{X} \cup \{s, t\}$, where the edges $E$ and edge capacities $u$ are defined as follows:
    \begin{itemize}
        \item For each $x \in \mathcal{X}$, there are edges $(s, x)$ and $(x,t)$ in $E$, with capacities $u(s, x) = \mathcal{D}^-(x)$ and $u(x, t) = \mathcal{D}^+(x)$.
        \item For all pairs $x, x' \in \mathcal{X}$ such that $x \to x'$, there is an edge $(x, x') \in E$ with capacity $u(x, x') = \infty$.
    \end{itemize}
\end{definition}

In terms of the graph defined above, computing the optimal classifier $f^*$ we seek is equivalent to finding a minimum $s$-$t$ cut on $G(\mathcal{D}, \mathcal{X})$. The intuition is that the edges from $s$ and to $t$ reflect the value gained from labeling an example $0$ or $1$, respectively; one of the edges must be cut, reflecting the loss of not assigning it to the corresponding side.  Moreover, if $x \to x'$, then the corresponding edge with infinite capacity prevents the assigning of $0$ to $x$ and $1$ to $x'$.
\begin{restatable}{theorem}{mincut} \label{thm:mincut}
    If $(S, \bar{S})$ is a minimum $s$-$t$ cut of $G(\mathcal{D}, \mathcal{X})$ (where $S$ is on the same side as $s$), then for the classifier $f^*(x) \coloneqq \mathbbm{1}(x \in \bar{S})$, we have $\ell_\mathcal{D}(f^*) = \min_{f} \ell_\mathcal{D}(f)$.
\end{restatable}

We note that, consequently, the optimal classifier can be computed in $\mathrm{poly}(|\mathcal{X}|)$ time. In practice, it is natural to expect that we do not know $\mathcal{D}$ exactly, but have a finite number of samples from it. A more practical option is to apply Theorem~\ref{thm:mincut} to the empirical distribution induced by the samples observed, and hope for the classifier computed from that to generalize to the true population distribution $\mathcal{D}$. 

\paragraph{Implementing {\sc Mincut}} Given a set $\widehat{\mathcal{X}}$ of $m$ i.i.d.~samples from $\mathcal{D}$, let $\widehat{\mathcal{D}}$ be the corresponding empirical distribution over $\widehat{\mathcal{X}}$, and $\bar{\mathcal{X}} \coloneqq \widehat{\mathcal{X}} \cup \{x': x' \to x, \exists x \in \widehat{\mathcal{X}} \}$. The {\sc Mincut} classifier is then obtained by applying Theorem \ref{thm:mincut} to $G(\widehat{\mathcal{D}},\widehat{\mathcal{X}})$, and extending it to $\bar{\mathcal{X}}$ as and when required. Here, note that {\sc Mincut} runs in time $\mathrm{poly}(m)$ (and not $\mathrm{poly}(|\mathcal{X}|)$), since $G(\widehat{\mathcal{D}},\widehat{\mathcal{X}})$ has $m$ nodes, and checking if a test point is in $\bar{\mathcal{X}}$ takes $\mathrm{poly}(m)$ time.

In light of traditional wisdom, the smaller $m$ is relative to $\mathcal{X}$, the larger the generalization error of {\sc Mincut} will be. We do not attempt a theoretical analysis in this regard, but note that when $\mathcal{X}$ is large, the generalization error can be extremely large (see Example 2 in the Supplement). The reason for this is two-fold: \begin{enumerate}
    \item {\sc Mincut} can give complicated decision boundaries.
    \item {\sc Mincut} is indecisive on samples not in $\bar{\mathcal{X}}$.\footnote{\scriptsize This is more likely to happen when using a large number of features.}
\end{enumerate}
Therefore, a suitable discretization of features is sometimes useful (see Section \ref{sec:experiments}). Note that {\sc Mincut} is truthful, by virtue of the infinite capacity edges in Definition \ref{def:G_dx}.
\section{Truthful classifiers and {\sc Hill-Climbing}} \label{sec:greedy}
The other drawback of {\sc Mincut}, related to the issue of generalization just discussed, is that it can be hard to interpret meaningfully in a practical setting. In this section, we devise a simpler alternative called {\sc Hill-Climbing}. To help introduce this algorithm, we first present a characterization of truthful classifiers in our setting, since we can limit our focus to them without loss of generality (as discussed in Section~\ref{sec:prelim}). For shorthand, we use the following definition:

\begin{definition}[$F'$-classifier] 
For a subset of features $F' \subseteq F$, a classifier $f$ is said to be an $F'$-classifier if for all $x \in \mathcal{X}$, we have $f(x) = f(x|_{F'})$, and if there exists $i \in F'$ such that $x_i = *$, then $f(x) = 0$.
\end{definition}
In other words, an $F'$-classifier depends only on the values of the features in $F'$, rejecting any $x$ where any of these is empty. We can collect many such classifiers into an ensemble as follows:

\begin{definition}[{\sc Max} Ensemble]
For a collection of classifiers $\mathcal{C} = \{f_j\}$, its {\sc Max} Ensemble classifier is given by $\textsc{Max}_{\mathcal{C}}(.) \coloneqq \max_j f_j(.)$.
\end{definition}
This is equivalent to getting each agent to pick the most favorable classifier from among those in $\{f_j\}$. Now using the above definitions we have the following characterization of truthful classifiers:

\begin{restatable}{theorem}{characterization}
\label{thm:characterization}
    A classifier $f$ is truthful iff $f(.) = \textsc{Max}_{\mathcal{C}}(.)$ for a collection of classifiers $\mathcal{C} = \{f_j\}$ such that, for some $\{F_j\} \subseteq 2^F$,
    each $f_j$ is an $F_j$-classifier . 
\end{restatable}

Now, for any truthful classifier $f$, we can bound the gap between its population loss $\ell_\mathcal{D}(f)$ and its empirical loss on a set of samples $\widehat{\mathcal{X}}$ denoted by $\ell_{\widehat{\mathcal{X}}}(f) \coloneqq \frac1m \sum_{i \in [m]} |f(x_i) - y_i|$. Before stating a theorem to this end, we define the following entities: Let $\mathcal{H}$ be a base hypothesis space over $\mathcal{X}$, and $n \in \{1, \dots, 2^k\}$ be a parameter. Define $d \coloneqq \dvc(\mathcal{H})$ to be the VC dimension of $\mathcal{H}$. Define $\bar{\mathcal{H}}$ as the set of all classifiers that can be written as the {\sc Max} Ensemble of $n$ classifiers in $\mathcal{H}$. 

\begin{restatable}{theorem}{generalization} \label{thm:generalization}
    Let $\widehat{\mathcal{X}} = \{(x_i, y_i)\}_{i \in [m]}$ be $m$ i.i.d. samples from $\mathcal{D}$. For any $f \in \bar{\mathcal{H}}$, for any $\delta > 0$, with probability at least $1 - \delta$, we have $\ell_\mathcal{D}(f) \le \ell_{\widehat{\mathcal{X}}}(f) + O\left(\sqrt{\frac{dn \cdot \log dn \cdot \log m + \log (1 / \delta)}{m}}\right)$.
\end{restatable}

It is easy to see that for any of the commonly used hypothesis spaces -- say $\mathcal{H}$ consists of linear hypotheses -- if a truthful classifier $f$ is in $\mathcal{H}$, then so are the components of the {\sc Max} Ensemble version of $f$ as in Theorem \ref{thm:characterization}. We have, however, stated Theorem \ref{thm:generalization} in slightly more general terms.

\paragraph{The {\sc Hill-Climbing} classifier:}
In what follows, we present a hill-climbing approach with provable convergence and generalization guarantees. The {\sc Hill-Climbing} classifier, henceforth referred to as {\sc HC}, is of the same form as given by the characterization of truthful classifiers in Theorem~\ref{thm:characterization}.\footnote{\scriptsize And, therefore, is truthful, and inherits Theorem \ref{thm:generalization}.} Intuitively, the approach works by considering a hierarchy of classifiers, organized by the features involved. For example, consider a setting with $k = 3$ features. We make a choice as to what classifiers we use --- say $f_1$ for input of the form $(x_1, *, *)$, $f_2$ for input of the form $(x_1, x_2, *)$, and $f_3$ for input of the form $(x_1, x_2, x_3)$. Any agent with features $1$ and $2$ (but not $3$), for example, should be able to report both features so as to be classified by $f_2$, or feature $2$ to be classified by $f_1$ instead. So in effect, assuming full knowledge of the classifiers, each agent can check all of the classifiers and choose the most favorable one. Without loss of generality, we assume that when a data point does not have all the features required by a classifier, it is automatically rejected.

\begin{algorithm}
   \caption{{\sc Hill-Climbing} (HC) Classifier}
   \label{alg:greedy}
\begin{algorithmic}
    \STATE {\bfseries Input:} data set $\widehat{\mathcal{X}} = \{(x_i, y_i)\}_{i \in [m]}$, n subsets $F_1, F_2, \ldots, F_n$ of F.
    \STATE Initialize: $t \gets 0$, $\{f_1^0, \ldots, f_n^0\}$.
    \WHILE{$\Delta > 0$}
        \FOR{$i = 1, 2, \ldots, n$}
            \STATE $S_i \gets \{(x, y) \in \widehat{\mathcal{X}}: f_j^t(x|_{F_j}) = 0, \forall j \neq i\}$.
            \STATE $f_i^{t+1} = \argmin_{f \in \mathcal{H}} \sum_{(x, y) \in S_i}  |f(x|_{F_i}) - y|$. \label{line-in-alg:agnostic-oracle}
        \ENDFOR
        \STATE $f^* \gets \textsc{Max}_{\{f^{t+1}_1, \ldots, f^{t+1}_n\}}$; $\ell_t = \ell_{\widehat{\mathcal{X}}}(f^*)$
        \STATE $\Delta \gets \ell_t - \ell_{t-1}$; $t \gets t + 1$
    \ENDWHILE
    
    \STATE {\bfseries Return:} $f^*$.
\end{algorithmic}
\end{algorithm}

In short, {\sc HC} (defined formally in Algorithm \ref{alg:greedy}) works as follows:  first choose a hypothesis space $\mathcal{H}$, in order for Theorem \ref{thm:generalization} to apply. Then select $n$ subsets of $F$ (where $n$ is a parameter), say $F_1, F_2, \ldots, F_n$. For each $F_j$, we learn a $F_j$-classifier, say $f_j$, from among those in $H$. Start by initializing these classifiers to any suitable $\{f^0_1, \ldots, f^0_n\}$. In each iterative step, each of the subclassifiers is updated to minimize the empirical loss on the samples that are rejected by all other classifiers. We next show that such an update procedure always converges. To do so, as far as our theoretical analysis goes, we assume we have black-box access to an agnostic learning oracle (Line \ref{line-in-alg:agnostic-oracle} in Algorithm \ref{alg:greedy}). After convergence, the {\sc HC} classifier is obtained as the {\sc Max} Ensemble of these classifiers. The generalization guarantee of Theorem \ref{thm:generalization} applies directly to the {\sc HC} classifier.

\begin{restatable}{theorem}{convergence} \label{thm:convergence}
  Algorithm \ref{alg:greedy} converges.  
\end{restatable}

\paragraph{Connection with {\sc Mincut}:}
The {\sc HC} formulation given above can be thought of as a less complicated version of {\sc Mincut}: some of the edge constraints are ignored with respect to learning the individual classifiers, and are instead factored in via the {\sc Max} function. Say $F_1 \subset F_2$. For some $x$, it is possible that $f_1(x|_{F_1}) = 1$ and $f_2(x|_{F_1}) = 0$. In other words, the individual classifiers could violate the {\sc Mincut} constraints, in order to learn classification functions that generalize well individually, and also collectively thanks to the combined {\sc HC} training procedure.

\paragraph{Implementing {\sc HC}:}
In practice, the classifiers $\{f_1, f_2, \ldots, f_n\}$ in {\sc HC} can be populated with any standard out-of-the-box methods such as logistic regression or neural networks, the choice of which can influence the performance of $f$. In Section 6, we test {\sc HC} with a few such options. The assumption of having access to an agnostic learning oracle does not play a crucial role in practice, with standard training methods performing well enough to ensure convergence.  
Also, {\sc HC} will converge in at most m (number of training examples) iterations, because in each iteration the number of correctly classified examples increases by at least one. (An iteration may need to train n individual classifiers.) This also means there is no difference between checking whether $\Delta > 0$ or $\Delta \ge 1/m$. In our experiments, we run {\sc HC} using $\Delta \ge 10^{-4}$, and convergence is achieved pretty quickly (see the Supplement for exact details).

\paragraph{Choosing subsets:} Note that we are free to choose any  $F_1, F_2, \ldots, F_n$ to define {\sc HC}. Its generalization (via Theorem \ref{thm:generalization}), will depend on the choice of $n$.
As more and more subsets of features are included (and further binning them based on their values), {\sc HC} starts behaving more and more like {\sc Mincut}. In addition, using a large number of subsets increases the computational complexity of {\sc HC}. In practice, therefore, the number of subsets must be limited somehow -- we find that some simple strategies like the following work reasonably well: 
\begin{enumerate*}[label=(\alph*)]
    \item selecting a few valuable features and taking all subsets of those features,
    \item taking all subsets of size smaller than a fixed number $k$, say $k=2$.
\end{enumerate*}
In many practical situations, a few features (possibly putting their values in just a few bins) are often enough to get close to optimal accuracy, also improving interpretability (e.g., see \citet{wang2015falling} or \citet{jung2017simple})
The question of devising a more nuanced algorithm for selecting these subsets merits a separate investigation, and we leave this to future work.

\section{Incentive-Compatible Logistic Regression}\label{sec:iclr}

As we just mentioned, it is challenging to directly apply {\sc HC} and {\sc Mincut} to a large number of features. As we will see, we can address this challenge in various ways to still get very strong performance with {\sc HC}. Moreover, HC enjoys remarkable generality, generalization and convergence guarantees. Nevertheless, we would like to have an algorithm that tries to make use of all the available features, while still being truthful. In this section, we present such an approach, which, as we show later in Section~\ref{sec:experiments}, indeed performs comparably to -- and in some cases better than -- {\sc Mincut} and {\sc HC}. 

Below we present a simple and truthful learning algorithm, Incentive-Compatible Logistic Regression ({\sc IC-LR}), which is a truthful variant of classical gradient-based algorithms for logistic regression.
Recall that in logistic regression, the goal is to learn a set of coefficients $\{\beta_i\}$, one for each feature $i \in F$, as well as an intercept $\beta_0$, such that for each data point $(x, y)$, the predicted label $\hat{y}$ given by
\[
    \hat{y} = \mathbbm{1} \left[ \sigma(\beta_0 + \sum_{i \in F} x_i \cdot \beta_i ) \ge 0.5 \right]
\]
fits $y$ as well as possible, where $\sigma(t) = 1 / (1 + e^{-t})$ is the logistic function.
Roughly speaking, {\sc IC-LR}. (formally defined in Algorithm~\ref{alg:iclr}) works by restricting the coefficients $\{\beta_i\}$ in such a way that dropping a feature (i.e., setting $x_i$ to $0$) can never make the predicted label larger.
If, without loss of generality, all feature values $x_i$ are nonnegative\footnote{\scriptsize If not, they can be suitably translated.}, then this is equivalent to: for each feature $i \in F$, the coefficient $\beta_i \ge 0$.
{\sc IC-LR}. enforces this nonnegativity constraint throughout the training procedure, by requiring a projection step after each gradient step, which projects the coefficients to the feasible nonnegative region by setting any negative coefficient to $0$ (equivalently, an $\ell_1$ projection).

\begin{algorithm}
   \caption{Incentive-Compatible Logistic Regression}
   \label{alg:iclr}
\begin{algorithmic}
    \STATE {\bfseries Input:} data set $\widehat{\mathcal{X}} = \{(x, y)\}$, learning rate $\{\eta_t\}$, $\delta \ge 0$.
    \STATE Initialize: $t \gets 0$, $\{\beta_0, \beta_1, \dots, \beta_k\}$.
    \WHILE{$\Delta > \delta$}
        \STATE $g_i \gets 0$ for all $i \in \{0, 1, \dots, k\}$
        \FOR{$(x, y) \in \widehat{\mathcal{X}}$}
            \STATE $g_0 \gets g_0 + \sigma\left(\beta_0 + \sum_{i \in F} x_i \cdot \beta_i\right) - y$
            \FOR{$i \in F$}
                \STATE $g_i \gets g_i + (\sigma\left(\beta_0 + \sum_{i \in F} x_i \cdot \beta_i\right) - y) \cdot x_i$
            \ENDFOR
        \ENDFOR
        \STATE $\forall i \in \{0, 1, \dots, k\}$, $\beta_i \gets \max\{\beta_i - \eta_t \cdot g_i, 0\}$
        \STATE $f^*(x) \coloneqq \mathbbm{1} \left( \sigma\left(\beta_0 + \sum_{i \in F} \beta_i \cdot x_i\right) \ge 0.5 \right)$
        \STATE $\ell_t = \ell_{\widehat{\mathcal{X}}}(f^*)$; $\Delta \gets \ell_t - \ell_{t-1}$; $t \gets t + 1$
    \ENDWHILE
    
    \STATE {\bfseries Return:} $f^*$.
\end{algorithmic}
\end{algorithm}

One potential issue with {\sc IC-LR}. is the following: if a certain feature $x_i \ge 0$ is negatively correlated with the positive classification label, then {\sc IC-LR} is forced to ignore it (since it is constrained to use positive coefficients). To make good use of this feature, we can include an inverted copy $x'_i  = \lambda - x_i$ (where $\lambda$ is chosen such that $x'_i \ge 0$). We could also choose an apt discretization of such features (using cross-validation) and translate the discretized bins into separate binary variables. Such a discretization can account for more complex forms of correlation, e.g., a certain feature's being too high or too low me makes the positive label likelier. In practice, we find that the latter method does better. If such transformations are undesirable, perhaps for reasons of complexity or interpretability, {\sc HC} methods are a safer bet.

\section{Evaluation} \label{sec:experiments}

In this section, we show that, when strategic withholding is at play, {\sc Mincut}, {\sc HC} and {\sc IC-LR} perform well and provide a significant advantage over several out-of-the-box counterparts (that do not account for strategic behavior).  
\paragraph{Datasets} Four credit approval datasets are obtained from the UCI repository \citep{Dua:2019}, one each from Australia, Germany, Poland and Taiwan. As is common for credit approval datasets, they are imbalanced to various degrees. In order to demonstrate the performance of classifiers in a standard, controlled setting, we balance them by random undersampling. There is a dedicated community \citep{chawla2004special} that looks at the issue of imbalanced learning. We do not delve into these issues in our paper, and evaluate our methods on both balanced and imbalanced datasets (see the Supplement for the latter). In addition, to demonstrate the challenge of high-dimensional data imposed on some of the classification methods, the experiments are run on the datasets  
\begin{enumerate*}[label=(\alph*)]
    \item restricted to 4 features,\footnote{\scriptsize According to ANOVA F-value evaluated before dropping any feature values.} and 
    \item with all available features.
\end{enumerate*} 
The basic characteristics of the datasets are summarized in Table \ref{tab:summary} -- note that there is enough variation in terms of the types of features present.
We then randomly remove a fraction $\epsilon = 0, 0.1, \dots, 0.5$ of all feature values in each dataset to simulate data that is missing ``naturally'' -- i.e., not due to strategic withholding. 
\begin{table}
\centering
\small
\caption{Data set summary statistics (num. = numerical, cat. = categorical)}
\begin{tabular}{lcccc}\toprule
Data set & Size & \makecell{Total \# of \\ features} &  \makecell{Size after \\ balancing}  & \makecell{Features after\\ restriction}\\ \midrule
Australia & 690 & 15 & 614  & 2 num., 2 cat. \\
Germany & 1000 & 20 & 600 & 1 num., 3 cat. \\
Poland & 5910 & 64 & 820  & 4 num. \\
Taiwan & 30,000 & 23 & 13,272  & 4 ordinal \\
\bottomrule
\end{tabular}
\label{tab:summary}
\end{table}

\paragraph{Testing} We test all methods under two ways of reporting: ``truthful'', i.e., all features are reported as is, and ``strategic'', i.e., some features might be withheld if it leads to a better outcome. We measure the test accuracy of each classifier, averaged over N=100 runs, with randomness over the undersampling and the data that is randomly chosen to be missing, to simulate data missing for non-strategic reasons. Other metrics, and details about implementing and training the classifiers, are discussed in the Supplement. It is important to note that for testing any method, we have to, in effect, compute the best response of each data point toward the classifier. Since the methods we propose are truthful, this is a trivial task. But for other methods, this might not be easy, thereby limiting what baselines can be used.





\paragraph{Classifiers} We evaluate our proposed methods, {\sc Mincut}, {\sc HC} with logistic regression ({\sc HC (LR)}) and neural networks ({\sc HC (ANN)}) as subclassifiers, and incentive-compatible logistic regression (IC-LR), against several baseline methods.

First, they will be compared against three out-of-the-box baseline classifiers: logistic regression ({\sc LR}), neural networks ({\sc ANN}) and random forest ({\sc RF}). We select {\sc LR} for its popularity in credit approval applications; we select {\sc ANN} for it being the best-performing individual classifier on some credit approval datasets \citep{lessmann2015benchmarking}; we select RF for it being the best-performing homogeneous ensemble on some credit approval datasets \citep{lessmann2015benchmarking}, as {\sc HC} can be viewed as a homogeneous ensemble method. For the sake of exposition, we present numbers just for baselines based on LR, as they perform relatively better. 

Second, for the purposes of comparison, we include {\sc Maj} -- predict the \emph{majority} label if examples with the exact same feature values appeared in the training set, and reject if not -- which can be thought of as a non-strategic counterpart of {\sc Mincut}. We also include k-nearest neighbors ({\sc kNN}) as a baseline, since it is closely related to {\sc Maj}.

These out-of-the-box classifiers need help dealing with missing data, whether they are missing naturally at training and test time or strategically at test time, and to this end, we employ 
\begin{enumerate*}[label = (\alph*)]
    \item {\sc Imp}: mean/mode imputation \citep{lessmann2015benchmarking}, and
    \item {\sc R-F}: reduced-feature modeling \citep{saar2007handling},
\end{enumerate*}  
for each of them.

When the dataset has a large number of features, {\sc Mincut} and {\sc IC-LR} can be directly applied. For {\sc HC}, we assist it in two ways: \begin{enumerate*}[label=(\alph*)]
    \item by selecting 4 features based on the training data, denoted by {\sc fs} (feature selection),\footnote{\scriptsize Such a technique can be applied to other methods too -- the results (see the Supplement) are not very different from those in Tables \ref{tab:accuracy,_2_bal,AAAI,FS,disc}.} and
    \item by choosing a limited number of small subsets (30 with 1 feature and 30 with 2 features), denoted by {\sc app} (approximation).
\end{enumerate*} Note that since our proposed methods are truthful, we can assume that features are reported as is. However, for all out-of-the-box classifiers, except {\sc Imp(LR)}, it is infeasible to simulate strategic withholding of feature values, due to the enormous number of combinations of features.

Last but not least, we test all methods with the discretization of continuous features (into categorical ones) \citep{DBLP:conf/ijcai/FayyadI93}, for reasons given in earlier sections.

\subsection{Results} 
For the sake of exposition, we report results only for $\epsilon = 0.2$. We also limit our exposition of {\sc HC}, {\sc Imp} and {\sc R-F} methods to those based on logistic regression, as these perform better than their {\sc ANN/RF/kNN} counterparts. For a comprehensive compilation of all results, along with standard deviation numbers, please refer to the Supplement.

\paragraph{With a small number of features (Table \ref{tab:accuracy,_2_bal,AAAI,FS}):} As expected, the out-of-the-box baselines perform well under truthful reporting, but not with strategic reporting. Our methods are robust to strategic withholding, and in line with the earlier discussion on the potential issues faced by {\sc Mincut} and IC-LR (in Sections \ref{sec:mincut} and \ref{sec:iclr}), we see that
\begin{enumerate*}[label = \textbf{(\alph*)}]
    \item {\sc HC(LR)} performs most consistently, and
    \item in some cases, {\sc Mincut} (e.g., Poland) and {\sc IC-LR} (e.g., Taiwan) do not do well.
\end{enumerate*}

\begin{table}
\centering
\caption{Our methods vs. the rest: mean classifier accuracy for $\epsilon=0.2$, balanced datasets, 4 features }
\fontsize{9pt}{9pt}\selectfont
\setlength{\tabcolsep}{4.3pt}
\begin{tabular}{@{}lcccccccc@{}}\toprule
\multirow{2}{*}[-3pt]{\makecell[l]{Classifier}} & \multicolumn{2}{c}{Australia} &  \multicolumn{2}{c}{Germany} &  \multicolumn{2}{c}{Poland} &  \multicolumn{2}{c}{Taiwan}\\
\cmidrule(l{5pt}r{5pt}){2-3} \cmidrule(l{5pt}r{5pt}){4-5} \cmidrule(l{5pt}r{5pt}){6-7} \cmidrule(l{5pt}r{5pt}){8-9}
& \multicolumn{1}{c}{Tru.} & \multicolumn{1}{c}{Str.} & \multicolumn{1}{c}{Tru.} & \multicolumn{1}{c}{Str.} & \multicolumn{1}{c}{Tru.} & \multicolumn{1}{c}{Str.} & \multicolumn{1}{c}{Tru.} & \multicolumn{1}{c}{Str.}\\ \midrule
{\sc HC(LR)} & .792  & {\bf.792} & .639  & .639 & .659  & .659 & .648 & .648\\
{\sc Mincut} & .770  & .770 & .580  & .580 & .501  & .501 & .652 & {\bf .652}\\

{\sc IC-LR} & .788  & .788 & .654  & {\bf .654} & .639  & .639 & .499 & .499\\
\hdashline

{\sc Imp(LR)} & .796  & .791 & {\bf .663}  & .580 & .{\bf 714}  & {\bf.660} & {\bf .670} & .618\\

{\sc R-F(LR)} & {\bf .808}  & .545 & .631  & .508 & .670  & .511 & .665 & .590\\
\bottomrule
\end{tabular}
\label{tab:accuracy,_2_bal,AAAI,FS}
\end{table}

\paragraph{With discretization (Table \ref{tab:accuracy,_2_bal,AAAI,FS,disc}):} As expected, discretization of numerical features into binary categories improves the performance of {\sc Mincut} and {\sc IC-LR}, for reasons explained in Sections \ref{sec:mincut} and \ref{sec:iclr} respectively. We also see some benefit from discretization for {\sc HC(LR)} when the features are mostly continuous (e.g., Poland), and less so when they are already discrete (e.g., Taiwan).

\begin{table}
\centering
\caption{Our methods vs. the rest: mean classifier accuracy for $\epsilon=0.2$, balanced datasets, 4 features (``w/ disc." stands for ``with discretization of features")}
\fontsize{9pt}{9pt}\selectfont
\setlength{\tabcolsep}{2.5pt}
\begin{tabular}{@{}lcccccccc@{}}\toprule
\multirow{2}{*}[-3pt]{\makecell[l]{Classifier}} & \multicolumn{2}{c}{Australia} &  \multicolumn{2}{c}{Germany} &  \multicolumn{2}{c}{Poland} &  \multicolumn{2}{c}{Taiwan}\\
\cmidrule(l{7pt}r{7pt}){2-3} \cmidrule(l{7pt}r{7pt}){4-5} \cmidrule(l{7pt}r{7pt}){6-7} \cmidrule(l{7pt}r{7pt}){8-9}
& \multicolumn{1}{c}{Tru.} & \multicolumn{1}{c}{Str.} & \multicolumn{1}{c}{Tru.} & \multicolumn{1}{c}{Str.} & \multicolumn{1}{c}{Tru.} & \multicolumn{1}{c}{Str.} & \multicolumn{1}{c}{Tru.} & \multicolumn{1}{c}{Str.}\\ \midrule
{\sc HC(LR)} w/ disc. & .794  & .794 & .641  & .641 & .692  & .692 & .650 & {\bf.650}\\
{\sc Mincut} w/ disc. & .789  & .789 & .629  & .629 & .692  & .692 & .649 & .649\\

{\sc IC-LR} w/ disc. & {\bf.800}  & {\bf .800} & .651  & {\bf.651} & .698  & {\bf .698} & .646 & .646\\\hdashline

{\sc Imp(LR)} w/ disc. & .799  & .762 & {\bf.652}  & .577 & {\bf.719}  & .631 & {\bf .686} & .541\\

{\sc R-F(LR)} w/ disc. & .796  & .542 & .633  & .516 & .708  & .522 & .684 & .587\\
\bottomrule
\end{tabular}
\label{tab:accuracy,_2_bal,AAAI,FS,disc}
\end{table}

\paragraph{With a large number of features (Table \ref{tab:accuracy,_2_bal,AAAI,NFS}):} 
We see broadly similar trends here, except that in the case with discretization, {\sc IC-LR} performs much better than before (e.g., Poland). The reason for this is that {\sc IC-LR} is able to use all the available features once they are discretized into binary categories. However, without discretization, {\sc HC} methods are more reliable (e.g., Poland and Taiwan).

\begin{table}
\centering
\caption{Our methods vs. the rest: mean classifier accuracy for $\epsilon=0.2$, balanced datasets, all features }
\fontsize{9pt}{9pt}\selectfont
\setlength{\tabcolsep}{1.9pt}
\begin{tabular}{@{}lcccccccc@{}}\toprule
\multirow{2}{*}[-3pt]{\makecell[l]{Classifier}} & \multicolumn{2}{c}{Australia} &  \multicolumn{2}{c}{Germany} &  \multicolumn{2}{c}{Poland} &  \multicolumn{2}{c}{Taiwan}\\
\cmidrule(l{7pt}r{7pt}){2-3} \cmidrule(l{7pt}r{7pt}){4-5} \cmidrule(l{7pt}r{7pt}){6-7} \cmidrule(l{7pt}r{7pt}){8-9}
& \multicolumn{1}{c}{Tru.} & \multicolumn{1}{c}{Str.} & \multicolumn{1}{c}{Tru.} & \multicolumn{1}{c}{Str.} & \multicolumn{1}{c}{Tru.} & \multicolumn{1}{c}{Str.} & \multicolumn{1}{c}{Tru.} & \multicolumn{1}{c}{Str.}\\ \midrule
{\sc HCfs(LR)}& .795  & .795 & .625  & .625 & .678  & .678 & .648 & .648\\
{\sc HCapp(LR)} & .777  & .777 & .617  & .617 & .658  & .658 & .638 & .638\\
{\sc Mincut} & .496  & .496 & .499  & .499 & .499  & .499 & .499 & .499\\
{\sc IC-LR} & .798  & .798 & .654  & {\bf .654} & .607  & .607 & .588 & .588\\
\hdashline

{\sc HCfs(LR)} w/ disc. & .794  & .794 & .632  & .632 & .694  & .694 & .649 & .649\\
{\sc HCapp(LR)} w/ disc. & .782  & .782 & .620  & .620 & .724  & .724 & .644 & .644\\
{\sc Mincut} w/ disc. & .534  & .534 & .503  & .503 & .499  & .499 & .550 & .550\\
{\sc IC-LR} w/ disc. & .805  & {\bf .805} & .653  & .653 & .773  & {\bf .773} & .667 & {\bf .667}\\\hdashline

{\sc Imp(LR)} & .802  & .701 & {\bf .663}  & .523 & .729  & .507 & .657 &.501 \\ 
{\sc Imp(LR)} w/ disc. & {\bf .809}  & .723 & .659  & .554 & {\bf .783}  & .503 & {\bf .697} & .501\\
\bottomrule
\end{tabular}
\label{tab:accuracy,_2_bal,AAAI,NFS}
\end{table}

\paragraph{On the out-of-the-box baselines:}
\begin{itemize*}
    \item \emph{Imputation-based methods} 
    are sensitive vis-\'a-vis the mean/mode values used. There is incentive to drop a certain feature if the imputed value is a positive signal. If there are many such features, then these methods perform poorly, as seen in Table \ref{tab:accuracy,_2_bal,AAAI,NFS} (cf. Table \ref{tab:accuracy,_2_bal,AAAI,FS}, Australia). 
    If the imputed values do not give a clear signal (e.g., when the distribution of each feature value is not skewed), there is a high variance in the performance of these methods (see the Supplement). In some cases, the benchmarks perform as well as, or slightly better than, our incentive-compatible classifiers. For example, in Table \ref{tab:accuracy,_2_bal,AAAI,FS}, for the Australia and Poland data sets, the accuracy of {\sc Imp(LR)} and that of {\sc HC(LR)} are within $0.001$ of each other. This happens because the imputed values are, in these cases (but not in most of our other cases), negative indicators of the positive label, and therefore there is generally no incentive to strategically drop features.
    \item \emph{Reduced-Feature modeling}, despite performing well under truthful reporting, allows too many examples to be accepted under strategic reporting, which hurts its performance. This is true especially for smaller $\epsilon$, as each subclassifier has fewer examples to train on, giving several viable options for strategic withholding.
\end{itemize*}

We note here that the variance (in the accuracy achieved) produced by our methods, since they are robust to strategic withholding, is much smaller than that of the baseline methods (exact numbers are deferred to the Supplement).

\newcommand{\ra}[1]{\renewcommand{\arraystretch}{#1}}

\section{Conclusion}

In this paper, we studied the problem of classification when each agent at prediction time can strategically withhold some of its features to obtain a more favorable outcome. We devised classification methods ({\sc Mincut}, {\sc HC} and {\sc IC-LR}) that are robust to this behavior, and in addition, characterized the space of all possible truthful classifiers in our setting. We tested our methods on real-world data sets, showing that they outperform out-of-the-box methods that do not account for the aforementioned strategic behavior.

 An immediate question that follows is relaxing the assumption of having access to truthful training data -- for example, one could ask what the best incentive-compatible classifier is given that the training data consists of best responses to a known classifier $f$; or, one could consider an online learning model where the goal is to bound the overall loss over time. A much broader question for future work is to develop a more general theory of robustness to missing data that naturally includes the case of strategic withholding.
\newpage

\section*{Acknowledgements}

We are thankful for support from NSF under award IIS-1814056.

\section*{Ethics Statement}

The methods presented in this paper are geared towards preventing the strategic withholding of data when machine learning methods are used in real-world applications. This will increase the robustness of ML techniques in these contexts: without taking this issue into account, deployment of these techniques will generally result in a rapid change in the distribution of submitted data due to the new incentives faced, causing techniques to work much more poorly than expected at training time. Thus, there is an AI safety \citep{Amodei2016ConcretePI} benefit to our work. The lack of strategic withholding also enables the collection of (truthful) quality data. Of course, there can be a downside to this as well if the data is not used responsibly, which could be the case especially if the features that (without our techniques) would have been withheld are sensitive or private in nature. 

The other issues to consider in our context are those of transparency and fairness. We assume that the classifier is public knowledge, and therefore, agents can appropriately best-respond. In practice, this might not be the case; however, agents may learn how to best-respond over time if similar decisions are made repeatedly (e.g., in the case of college admissions or loan applications). While US college admission is often a black box, it need not be; many countries have transparent public criteria for university admissions (e.g., the Indian IIT admission system), and the same is true in many other contexts (e.g., Canadian immigration). Of course, transparency goes hand in hand with interpretability, i.e., the classifier must be easily explainable as well, and there could be a trade-off, in principle, between how easy the classifier is to interpret and the accuracy it can achieve. It is also possible that our methods hurt the chances of those with more missing data (similarly to how immigrants without credit history may not be able to get a credit card). This is to some extent inevitable, because if one can get in without any feature, everyone could get in by dropping all features. Therefore, the issue of fairness might arise in the case where some groups systematically tend to have more missing data.

\medskip

\bibliographystyle{plainnat}
\bibliography{main.bib}
\appendix
\onecolumn

\section{Further related work}\label{app:related}
 Our work can, in a way, be thought of as studying an adversarial classification \citep[see, e.g.,][]{vorobeychik2018adversarial} problem 
 -- in particular, a decision-time, white-box, targeted attack on binary classifiers, assuming that the only strategy available to the attacker is to remove feature values, and the attacker's goal is to maximize the number of instances classified as positive. In this regard, what we study is similar in spirit to some of the existing literature \citep{globerson2006nightmare,syed2010semi,dekel2010learning} on adversarial classification. 
 
 For example, \cite{globerson2006nightmare} consider a problem where, at test time, the attacker can set up to a certain number of features (say pixels in an image) to zero for each instance individually in a way that is most harmful to the classifier chosen. To be robust to such attacks, they devise convex programming based methods that avoid depending on small sets of features to learn the class structure. Our work is different in that we take a more game-theoretic approach to designing classifiers (including ensemble-based ones) that are fully resistant to the strategic withholding of features by agents (that prefer being labeled positively). Moreover, we make no assumptions on the actual structure of the feature space. 
 
 Our work is also related to the literature on classification withmissing data (Batista and Monard 2003; Marlin 2008).  We devise methods that can deal with the strategic withholding by agents of some of their features, against a backdrop of missing data caused by natural reasons (e.g., nonstationary distributions of which features are present, or input sensor failures). The {\sc HC} method can be viewed as an ensemble method for missing data \citep{conroy2016dynamic} that is strategy-proof against the aforementioned strategic behavior. We also study the performance of other standard, non-strategic classification methods for missing data in the strategic setting, including predictive value imputation and reduced-feature modeling \citep{saar2007handling}.



\section{Proofs}
\mincut*
\begin{proof}
    First observe that any classifier $f$ can be viewed equivalently as a subset of $\mathcal{X}$, given by
    \[
        \{x \in \mathcal{X} \mid f(x) = 1\}.
    \]
    Below, we use these interpretations, i.e., as a function or a subset, of a classifier interchangeably.
    
    The loss of a truthful classifier $f$ can then be written as
    \begin{align*}
        \ell_\mathcal{D}(f) & = \Pr_{(x, y) \sim \mathcal{D}}[(x \in f \wedge y = 0) \vee (x \notin f \wedge y = 1)] \\
        & = \Pr_{(x, y) \sim \mathcal{D}}[x \in f \wedge y = 0] + \Pr_{(x, y) \sim \mathcal{D}}[x \notin f \wedge y = 1] \\
        & = \sum_{x \in f} \Pr_{(x', y) \sim \mathcal{D}}[x' = x \wedge y = 0] + \sum_{x \notin f} \Pr_{(x', y) \sim \mathcal{D}}[x' = x \wedge y = 1] \\
        & = \sum_{x \in f} \mathcal{D}^-(x) + \sum_{x \notin f} \mathcal{D}^+(x),
    \end{align*}
    where the last line follows from \cref{def:Dplusminus}.
    
    Therefore, our goal is to solve the following optimization problem:
    \[
        \begin{aligned}
            \min_{f \subseteq \mathcal{X}} \qquad & \sum_{x \in f} \mathcal{D}^-(x) + \sum_{x \notin f} \mathcal{D}^+(x) & \\
            \text{s.t.} \qquad & x' \in f \Longrightarrow x \in f & \forall x, x' \in \mathcal{X}\text{ where } x \to x'.
        \end{aligned}
    \]
    
    Consider the following min-cut formulation (using \cref{def:G_dx}):
    Let $G(\mathcal{D}, \mathcal{X})$ be a directed capacitated graph, with vertices $V = \mathcal{X} \cup \{s, t\}$, with edges $E$ and edge capacities $u$ defined as follows:
    \begin{itemize}
        \item For each $x \in \mathcal{X}$, there is an edge $(s, x) \in E$ with capacity $\mathcal{D}^-(x)$, and an edge $(x, t) \in E$ with capacity $\mathcal{D}^+(x)$.
        \item For each pair $x, x' \in \mathcal{X}$ where $x \to x'$, there is an edge $(x', x) \in E$ with capacity $\infty$.
    \end{itemize}
    Observe that each finite-capacity $s$-$t$ cut $(S, \bar{S})$ corresponds bijectively to a truthful classifier $f \coloneqq \mathbbm{1}(x \in \bar{S} \setminus \{t\})$.
    Moreover, the capacity of the cut is given precisely by
    \[
        \sum_{x \in \bar{S} \cap \mathcal{X}} \mathcal{D}^-(x) + \sum_{x \in S \cap \mathcal{X}} \mathcal{D}^+(x) = \sum_{x \in f} D^-(x) + \sum_{x \notin f} D^+(x) = \ell_\mathcal{D}(f).
    \]
    Therefore, any $s$-$t$ min-cut corresponds to an optimal classifier $f^*$, which can be computed ``efficiently'' (i.e., in time $\mathrm{poly}(|\mathcal{X}|)$) using any efficient max-flow algorithm given complete knowledge of $\mathcal{D}$.
\end{proof}

\characterization*
\begin{proof}
    Recall that the revelation principle holds in our setting (as mentioned in Section Section \ref{sec:prelim}, also see Proposition \ref{prop:revelation}).
    It therefore suffices to characterize all direct revelation classifiers.
    For any (not necessarily truthful) classifier $f$, consider its direct revelation implementation $f'$, which maps feature values $x$ to the most desirable label the data point can get by dropping features, i.e.,
    \[
        f'(x) = \max_{x': x \to x'} f(x').
    \]
    We argue below that $f'$ has the desired form.
    
    Observe that depending on which features a data point $x$ has, $f$ can be decomposed into $2^k$ subclassifiers, denoted $\{f_F\}_{F \subseteq [k]}$.
    The label of $x$ is then determined in the following way: let $F_x$ be the set of features possessed by $x$, i.e.,
    \[
        F_x = \{i \in [k] \mid x_i \ne *\}.
    \]
    Then
    \[
        f(x) = f_{F_x}(x).
    \]
    Moreover, observe that (1) $f_F$ effectively depends only on $x|_F$ (i.e., $f_F(x) = f_F(x|_F)$), since $f_F$ only acts on those data points where all features not in $F$ are missing, and (2) without loss of generality, $f_F$ rejects any data point with a missing feature $i \in F$, since $f_F$ never acts on a data point where such a feature $i \in F$ is missing.
    Now consider how $f'$ works on a data point $x$.
    For any $F \subseteq F_x$, by dropping all features not in $F$, $x$ can report $x_|F$.
    Moreover, for any such $F \subseteq F_x$, $f(x|_F) = f_F(x|_F)$.
    $f'$ outputs $1$ for $x$, iff there exists $F \subseteq F_x$, such that $f_F(x|_F) = 1$.
    One can therefore write $f'$ in the following way: for any $x \in \mathcal{X}$,
    \[
        f'(x) = \max_{F \subseteq [k]} f_F(x),
    \]
    as desired.
\end{proof}

\generalization*
\begin{proof}
    Recall that $\bar{\mathcal{H}}$ is defined as the set of all classifiers that can be written as the MAX Ensemble of $n$ classifiers in $\mathcal{H}$. Given the classical VC inequality \citep[e.g.,][Theorem 6.11]{shalev2014understanding}, we only need to bound the VC dimension of $\mathcal{\bar{H}}$, and show that
    \[
        \dvc(\bar{\mathcal{H}}) = O(dn \cdot \log dn),
    \]
    where $d$ is the VC dimension of $\mathcal{H}$.
    To this end, observe that each $f \in \bar{\mathcal{H}}$ is essentially a decision tree with $n + 1$ leaves, where each leaf is associated with a binary label, and each internal node corresponds to a classifier in $\mathcal{H}$.
    To be precise, $f$ can be computed in the following way: for any $x \in \mathcal{X}$, if $f_1(x) = 1$, then $f(x) = 1$; otherwise, if $f_2(x) = 1$, then $f(x) = 1$, etc.
    It is known \citep[see,~e.g.,][Section~5.2]{daniely2015multiclass} that the class of all such decision trees with $n + 1$ leaves, which is a superset of $\bar{\mathcal{H}}$, has VC dimension $O(dn \log dn)$.
    As a result, $\dvc(\mathcal{H}) = O(dn \log dn)$, and the theorem follows.
\end{proof}

\convergence*
\begin{proof}
Given $f = \textsc{Max}_{\{f_1^t, f_2^t, \ldots, f_n^t\}}$, consider a single update step for, say, $f_1^{t}$. 
As in Algorithm \ref{alg:greedy}, define:
\begin{align*}
    S_1 &= \{(x, y) \in \widehat{\mathcal{X}}: f_j^t(x|_{F_j}) = 0, \forall j \neq 1\}, \\
    S_{-1} &= S \setminus S_1.
\end{align*}
Then we perform the update as follows:
\begin{align*}
    f_1^{t+1} = \argmin_ {h \in \mathcal{H}} \sum_{(x, y) \in S_1} |h(x|_{F_1}) - y|.
\end{align*}
Let $f' = \textsc{Max}_{\{f_1^{t+1}, f_2^t, \ldots, f_n^t\}}$. Now, the loss calculated for $f'$ is
\begin{align*}
    \ell_{\widehat{\mathcal{X}}}(f') &= \frac{1}{m} (|S_1| \cdot \ell_{S_1}(f') + |S_{-1}| \cdot \ell_{S_{-1}}(f')) \\
    &= \frac{1}{m} (|S_1| \cdot \ell_{S_1}(f_1^{t + 1}) + |S_{-1}| \cdot \ell_{S_{-1}}(f')) \\
    &= \frac{1}{m} (|S_1| \cdot \ell_{S_1}(f_1^{t + 1}) + |S_{-1}| \cdot \ell_{S_{-1}}(f)) \\
    &\le \frac{1}{m} (|S_1| \cdot \ell_{S_1}(f_1^t) + |S_{-1}| \cdot \ell_{S_{-1}}(f)) \\
    &= \ell_{\widehat{\mathcal{X}}}(f).
\end{align*}
The inequality in the above sequence of steps follows from the fact that $f_j^{t+1}$ accrues a lower loss on $S_1$ than $f_j^{t}$ by definition, and that the classification outcomes for any $(x, y) \in S_{-1}$ is the same for $f$ and $f'$.

If we treat $\ell_{\widehat{\mathcal{X}}}(f)$ as a potential function, we can see that it can only decrease with each step, and therefore, the algorithm has to converge at some point.
\end{proof}
\section{Revelation Principle}

There are many results in the literature on partial verification as to the validity of the revelation principle in various settings \citep{green1986partially,yu2011mechanism,kephart2015complexity,kephart2016revelation}. For our purposes, as mentioned in Section \ref{sec:prelim}, when the reporting structure is given by a partial order (i.e., it is transitive, meaning for any $x_1, x_2, x_3$, $x_1 \to x_2 \text{ and } x_2 \to x_3 \implies x_1 \to x_3$), the revelation principle holds. Below we give a quick proof for why this is the case in our setting. 
\begin{proposition}\label{prop:revelation}
    For any classifier $f: \mathcal{X} \to \{0, 1\}$, there is a truthful classifier $f'$ such that after misreporting, $f$ and $f'$ output the same label for all $x \in \mathcal{X}$, i.e.,
    \[
        f'(x) = \max_{x': x \to x'} f(x').
    \]
\end{proposition}
\begin{proof}
    Below we explicitly construct $f'$.
    Let $f'$ be such that for $x \in \mathcal{X}$,
    \[
        f'(x) = \max_{x': x \to x'} f(x).
    \]
    Clearly $f'$ and $f$ output the same label after strategic manipulation.
    We only need to show $f'$ is truthful, i.e., for any $x_1, x_2 \in \mathcal{X}$ where $x_1 \to x_2$, $f'(x_1) \ge f'(x_2)$.
    Let $X_1 = \{x': x_1 \to x'\}$ and $X_2 = \{x': x_2 \to x'\}$.
    Recall that $\to$ is transitive and $x_1 \to x_2$, so $X_1 \supseteq X_2$.
    Now we have
    \[
        f'(x_1) = \max_{x \in X_1} f(x) \ge \max_{x \in X_2} f(x) = f'(x_2). \qedhere
    \]
\end{proof}

Note that the above proof crucially depends on transitivity of the reporting structure.
In fact, if the reporting structure $\to$ is not transitive, then the revelation principle in general does not hold.
To see why this is the case, suppose $\to$ is not transitive, and let $x_1, x_2, x_3$ be such that $x_1 \to x_2$, $x_2 \to x_3$, and $x_1 \not\to x_3$.
Suppose we want to assign label $0$ to $x_1$, and label $1$ to $x_2$ and $x_3$, then the only way to achieve that is to implement a classifier $f$ where $f(x_1) = f(x_2) = 0$ and $f(x_3) = 1$.
However, this classifier is not truthful, since $x_2$ always misreports as $x_3$ in order to be accepted.

\section{Other observations}
\subsection{Regarding {\sc Mincut}}\label{app:mincut}
Naturally, the test error of {\sc Mincut} depends on $\mathcal{X}$ and $m$. For example, If $\mathcal{X}$ is discrete and small, one would expect that {\sc Mincut} is almost optimal given enough samples. However, when $\mathcal{X}$ is large or even infinite, the generalization gap can be extremely large. To see why this is true, consider the following example:

\begin{example}\label{ex:mincut-generalization}
Say we are given a feature space with two features, each of which can take any real value between $0$ and $1$. Let the marginal distribution of $\mathcal{D}$ on $\mathcal{X}$ be the uniform distribution over $\mathcal{X} = \{(x, y), (x, *), (*, y) \mid x, y \in [0, 1]\}$.
When we see a new data point $(x, y)$, unless we already have $(x, y)$, $(x, *)$ or $(*, y)$ in the set of samples (which happens with probability $0$), we know absolutely nothing about the label of $(x, y)$, and therefore by no means we can expect $f'$ to predict the label of $(x, y)$ correctly --- in fact, $f'$ will always assign label $0$ to such a data point.
\end{example}

\subsection{On truthful classifiers and hill-climbing}
Below, we make a few remarks regarding the generalization bound (\cref{thm:generalization}) for {\sc HC}.
\begin{itemize}
    \item Observe that the generalization gap depends polynomially on the number of subclassifiers $n$.
    Without additional restrictions, $n$ can be as large as $2^k$ leading to a gap which is exponential in $k$.     This suggests that in practice, to achieve any meaningful generalization guarantee, one has to restrict the number of subclassifiers used. In fact, we do run our algorithm on a small set of features in Section \ref{sec:experiments}.
    \item Recall that the class of linear classifiers in the $k$-dimensional Euclidean space has VC dimension $k + 1$.
    So, if we restrict all subclassifiers to be linear, and require that the number of subclassifiers $n$ to be constant, then Theorem~\ref{thm:generalization} implies that with high probability, the generalization gap is
    \[
        \widetilde{O}\left(\sqrt{\frac{k}{m}}\right),
    \]
    where $k$ is the number of features, $m$ is the number of samples, and $\widetilde{O}$ hides a logarithmic factor. Our algorithms are practicable in this kind of regime.
\end{itemize}

\section{Experiments}
\subsection{Implementation details}

In our implementation, we use Python's Scikit-learn (0.22.1) package \citep{scikit-learn} of classifiers and other machine learning packages whenever possible. The categorical features in the datasets are one-hot encoded. To help ensure the convergence of gradient-based classifiers, we then standardize features by removing the mean and scaling to unit variance.

For imputation-based classifiers, we use mean/mode imputation: mean for numerical and ordinal features, and mode for categorical features. For reduced-feature-based classifiers, we default to reject if the test data point's set of available features was unseen in the training process. For classification methods involving  Fayyad and Irani's MDLP discretization algorithm, we use a modified version of \texttt{Discretization-MDLPC}, licensed under GPL-3\footnote{\texttt{Discretization-MDLPC} codebase: \url{https://github.com/navicto/Discretization-MDLPC}.}.

The performance of each classifier under each setting is evaluated with Nx2-fold cross-validation \citep{dietterich1998approximate}: training on 50\% of the data and testing on the remaining 50\%; repeat N $=100$ times. To tune the parameters for the classifiers, we perform grid search over a subset of the parameter space considered by \cite{lessmann2015benchmarking}, in a 5-fold cross-validation on every training set of the (outer) Nx2-cross-validation loop. 


\subsection{Additional experimental results}

We evaluate our methods on datasets with 1) 4 selected features, same for all runs (Table \ref{tab:AAAI,accuracy,_0_bal,preFS} to \ref{tab:AAAI,AUC,_5,preFS}), 2) 4 selected features, based on the training data at each run (Table \ref{tab:AAAI,accuracy,_0_bal,FS} to \ref{tab:AAAI,AUC,_5,FS}), and 3) all available features (Table \ref{tab:AAAI,accuracy,_0_bal,NFS} to \ref{tab:AAAI,AUC,_5,NFS}). In addition, we evaluate our methods both with and without balancing the datasets through random undersampling. This is denoted by ``balanced datasets'' when we undersample before the experiment, and ``unbalanced datasets'' when we do not. We vary $\epsilon$ from 0 to .5 and report classifier accuracy and AUC (when applicable).

Comparing Figure \ref{fig:box_preFS_BAL_ACC} and \ref{fig:box_preFS_UNBAL_ACC}, it appears that there is no significant difference in relative accuracy across the various methods when applied to balanced and imbalanced datasets. However, comparing, for example, Table \ref{tab:AAAI,accuracy,_2_bal,preFS} and \ref{tab:AAAI,accuracy,_2,preFS}, we observe that when the dataset is highly unbalanced, classifiers based on imputation and reduced-feature modeling, when faced with strategic reporting, tend to accept everything and yield a considerably high accuracy. Many other general issues regarding the use of accuracy as a metric on unbalanced datasets are known \citep{japkowicz2000aaai,chawla2003icml,chawla2004special}. In practice, thresholding methods are sometimes used to determine a proper threshold for binary prediction in such cases \citep{lessmann2015benchmarking,elkan2001foundations}.

Therefore, in addition to accuracy, we evaluate our approach with area under the receiver operating characteristic curve (AUC). AUC becomes a useful metric when doing imbalanced classification because often a balance of false-positive and false-negative rates is desired, and it is invariant to the threshold used in binary classification.\footnote{... although AUC's global perspective assumes implicitly that all thresholds are equally probable. This is often criticized as not plausible in credit scoring \citep{hand2013area}} For {\sc Mincut}, its receiver operating characteristic curve is undefined because it does not output probabilistic predictions; for {\sc Hill-Climbing}, we take the maximum of the probabilistic predictions across all applicable classifiers to be the {\sc Hill-Climbing} classifier's probabilistic prediction for a data point, and obtain AUC from that. From 
Table \ref{tab:AAAI,AUC,_0,preFS} to \ref{tab:AAAI,AUC,_5,preFS}, we observe that our three proposed methods generally yield a AUC as good as, if not higher than, imputation- and reduced-feature-based classifiers on imbalanced datasets (also Figure \ref{fig:box_preFS_UNBAL_AUC}). The same holds for balanced datasets too (\Cref{fig:box_preFS_BAL_AUC}).

For completeness, we also include the performance of classifiers based on imputation and reduced-feature modeling, with discretization. As expected, common classifiers are generally less prone to overfitting than {\sc Mincut}, and discretizing the feature space only limits their performance.

The number of iteration the training of {\sc Hill-Climbing} classifiers takes to converge varies by dataset, but usually in no more than than 5 passes through all the subclassifiers. For the balanced Australia dataset, for example, {\sc HC(LR)} takes an average of 4.28 iterations to converge (median 3); {\sc HC(LR)} w/ disc. takes an average of 2.45 (median 2); {\sc HC(ANN)}: takes an average of 3.56 (median 3); {\sc HC(ANN)} w/ disc. takes an average of 2.62 (median 2).

\begin{figure}[ht!]
  \caption{Selected classifier accuracy w/ strategic behavior, balanced datasets, pre-selected 4 features}
  \includegraphics[width=\textwidth]{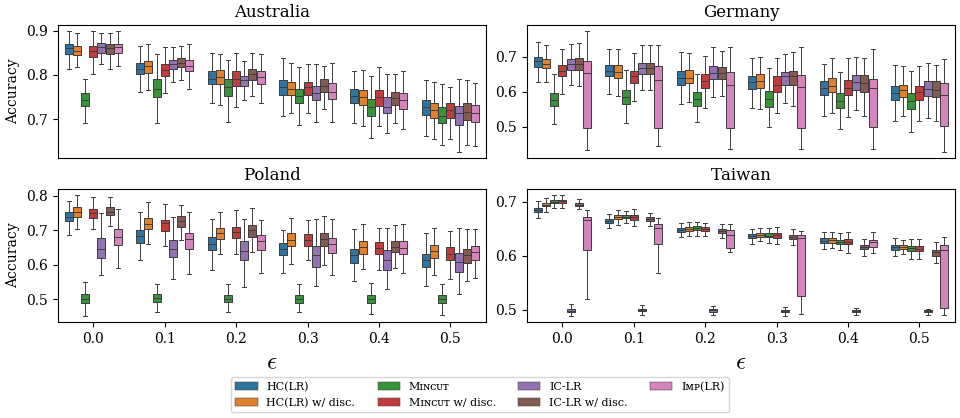}
  \label{fig:box_preFS_BAL_ACC}
\end{figure}
\begin{figure}[ht!]
  \caption{Selected classifier AUC w/ strategic behavior, balanced datasets, pre-selected 4 features}
  \includegraphics[width=\textwidth]{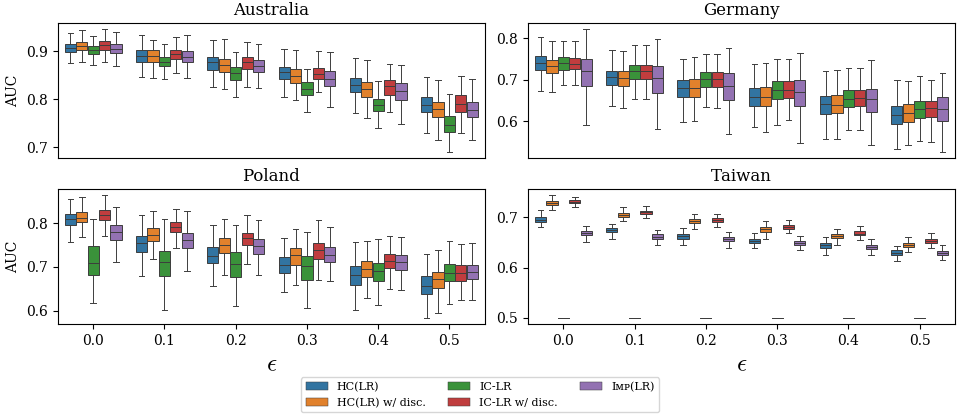}
  \label{fig:box_preFS_BAL_AUC}
\end{figure}
\begin{figure}[ht!]
  \caption{Selected classifier accuracy w/ strategic behavior, unbalanced datasets, pre-selected 4 features}
  \includegraphics[width=\textwidth]{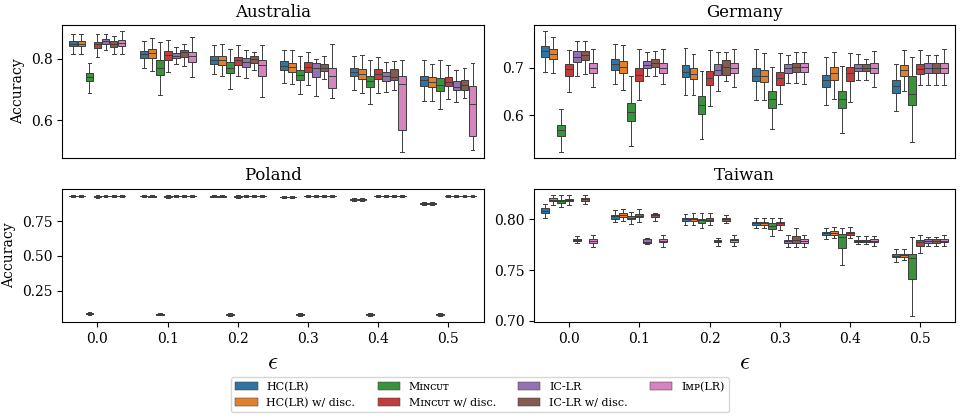}
  \label{fig:box_preFS_UNBAL_ACC}
\end{figure}
\begin{figure}[ht!]
  \caption{Selected classifier AUC w/ strategic behavior, unbalanced datasets, pre-selected 4 features}
  \includegraphics[width=\textwidth]{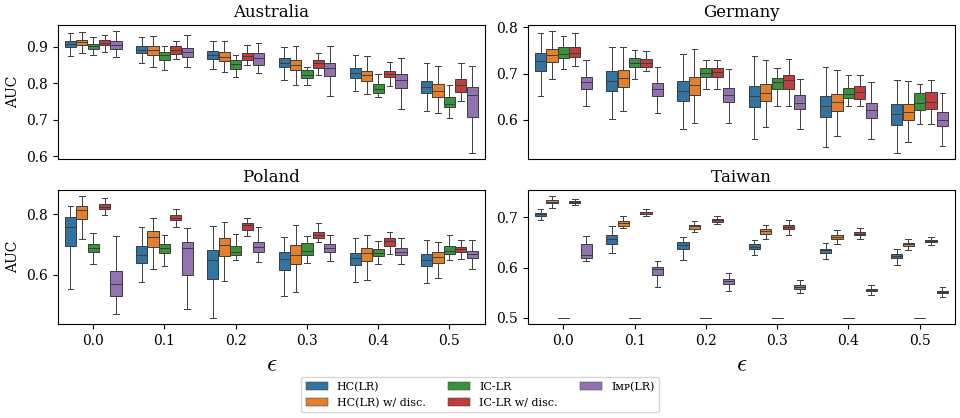}
  \label{fig:box_preFS_UNBAL_AUC}
\end{figure}

\begin{figure}[ht!]
  \caption{Selected classifier accuracy w/ strategic behavior, balanced datasets, each classifier selects 4 best features}
  \includegraphics[width=\textwidth]{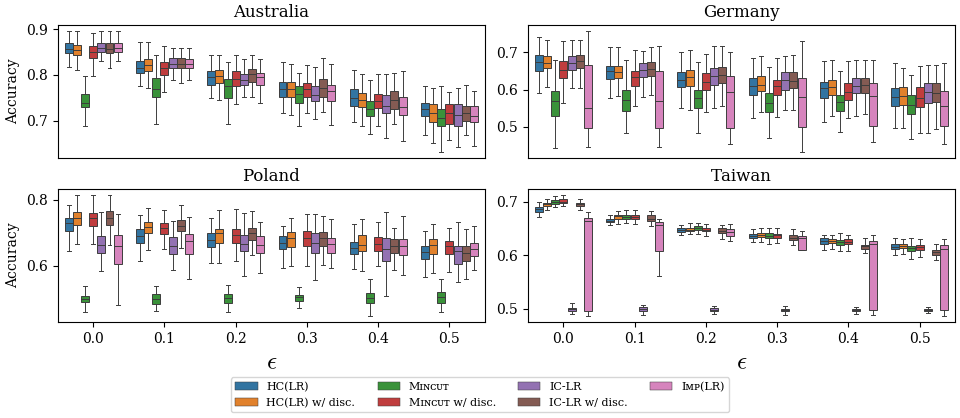}
  \label{fig:box_FS_BAL_ACC}
\end{figure}
\begin{figure}[ht!]
  \caption{Selected classifier AUC w/ strategic behavior, balanced datasets, each classifier selects 4 best features}
  \includegraphics[width=\textwidth]{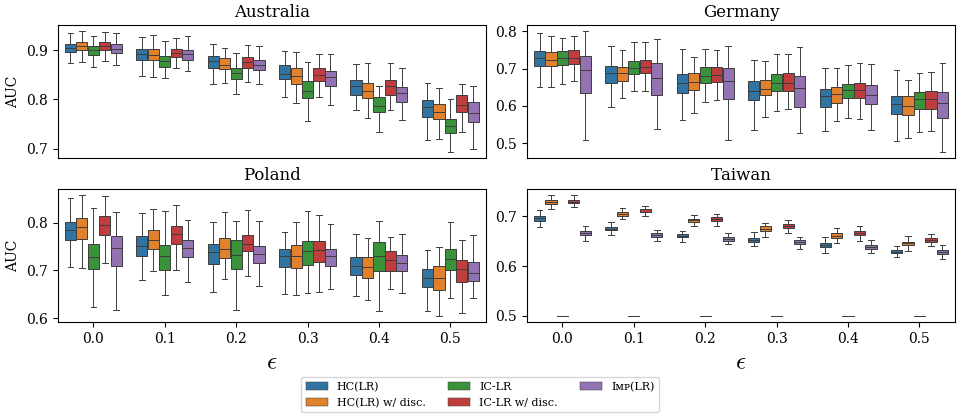}
  \label{fig:box_FS_BAL_AUC}
\end{figure}
\begin{figure}[ht!]
  \caption{Selected classifier accuracy w/ strategic behavior, unbalanced datasets, each classifier selects 4 best features}
  \includegraphics[width=\textwidth]{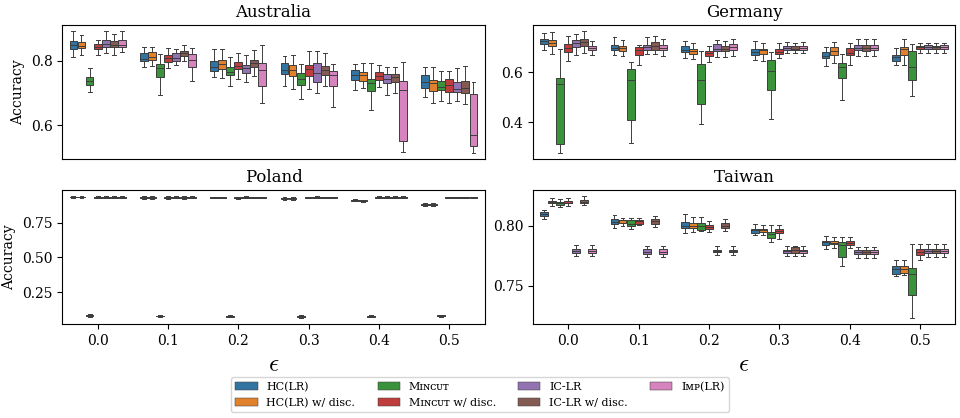}
  \label{fig:box_FS_UNBAL_ACC}
\end{figure}
\begin{figure}[ht!]
  \caption{Selected classifier AUC w/ strategic behavior, unbalanced datasets, each classifier selects 4 best features}
  \includegraphics[width=\textwidth]{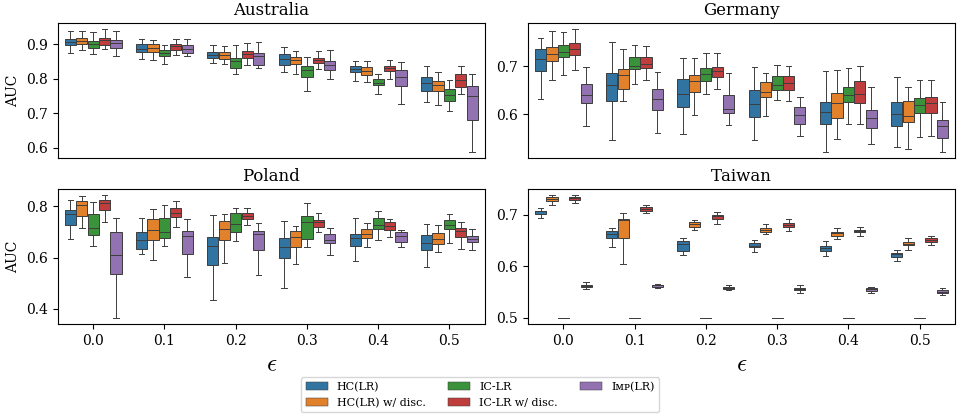}
  \label{fig:box_FS_UNBAL_AUC}
\end{figure}

\begin{figure}[ht!]
  \caption{Selected classifier accuracy w/ strategic behavior, balanced datasets, all features}
  \includegraphics[width=\textwidth]{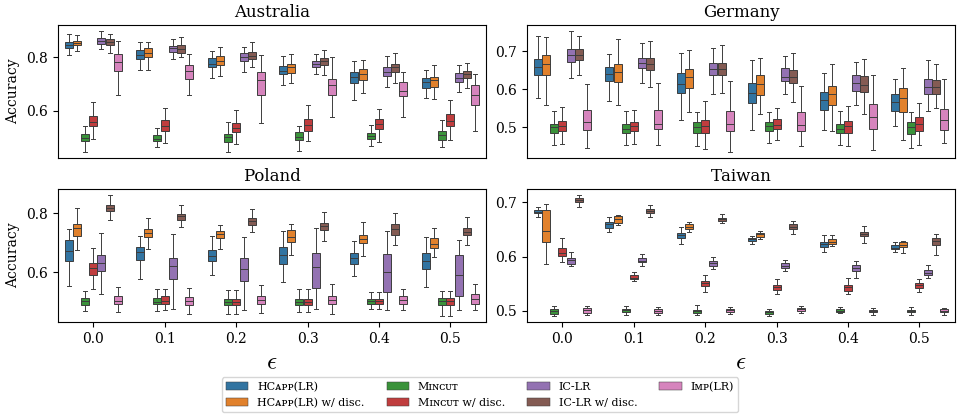}
  \label{fig:box_NFS_BAL_ACC}
\end{figure}\begin{figure}[ht!]
  \caption{Selected classifier AUC w/ strategic behavior, balanced datasets, all features}
  \includegraphics[width=\textwidth]{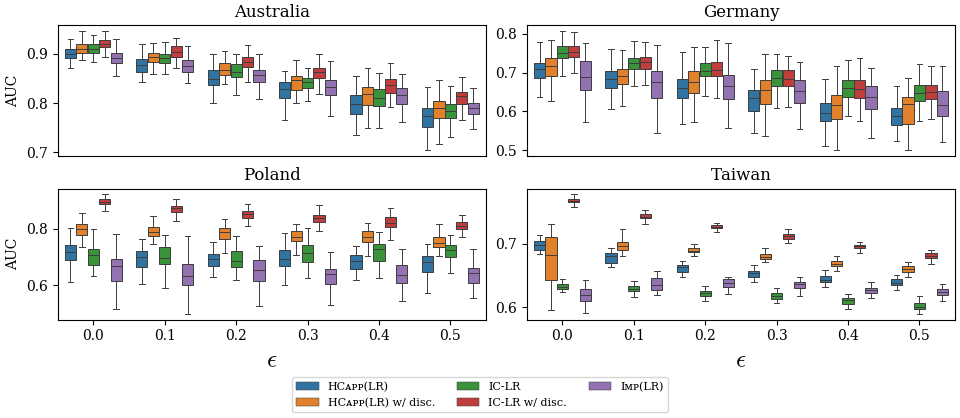}
  \label{fig:box_NFS_BAL_AUC}
\end{figure}\begin{figure}[ht!]
  \caption{Selected classifier accuracy w/ strategic behavior, unbalanced datasets, all features}
  \includegraphics[width=\textwidth]{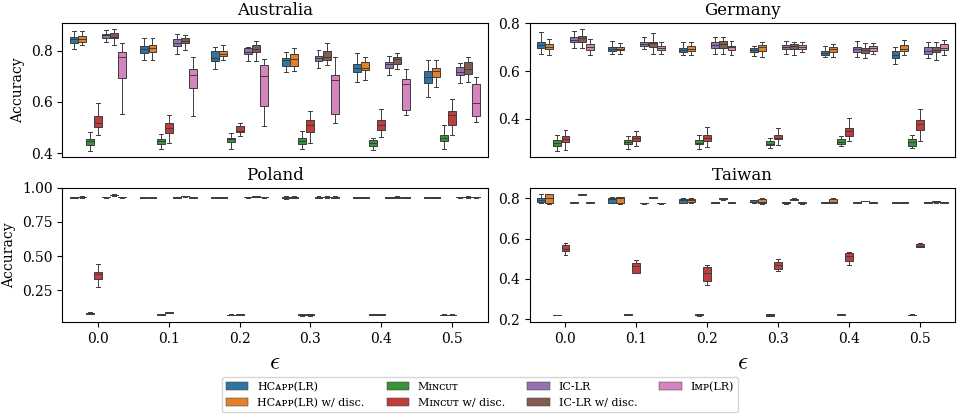}
  \label{fig:box_NFS_UNBAL_ACC}
\end{figure}\begin{figure}[ht!]
  \caption{Selected classifier AUC w/ strategic behavior, unbalanced datasets, all features}
  \includegraphics[width=\textwidth]{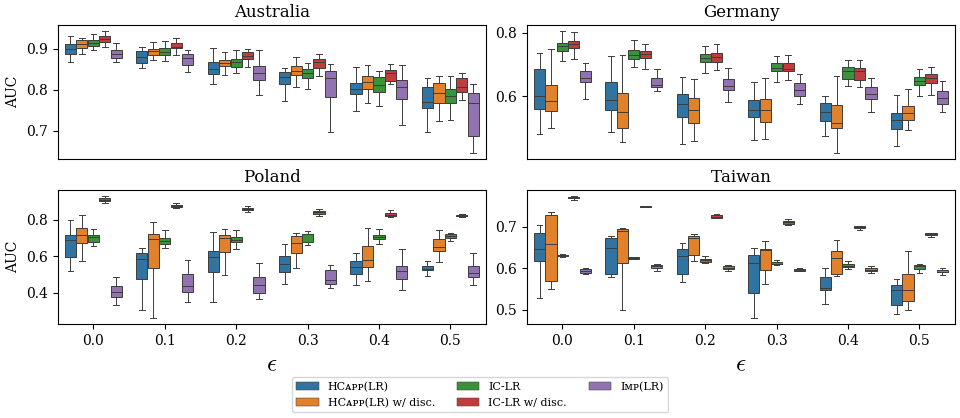}
  \label{fig:box_NFS_UNBAL_AUC}
\end{figure}



\clearpage




\begin{table*}\centering
\caption{Our methods vs. the rest: mean classifier accuracy for $\epsilon=0.0$, balanced datasets, pre-selected 4 features.}
\setlength{\tabcolsep}{2.7pt}
\begin{tabular}{@{}lcccccccc@{}}\toprule
\multirow{2}{*}[-3pt]{\makecell[l]{Classifier}} & \multicolumn{2}{c}{Australia} & \multicolumn{2}{c}{Germany} & \multicolumn{2}{c}{Poland} & \multicolumn{2}{c}{Taiwan}\\
\cmidrule(l{7pt}r{7pt}){2-3} \cmidrule(l{7pt}r{7pt}){4-5} \cmidrule(l{7pt}r{7pt}){6-7} \cmidrule(l{7pt}r{7pt}){8-9}
& \multicolumn{1}{c}{Tru.} & \multicolumn{1}{c}{Str.} & \multicolumn{1}{c}{Tru.} & \multicolumn{1}{c}{Str.} & \multicolumn{1}{c}{Tru.} & \multicolumn{1}{c}{Str.} & \multicolumn{1}{c}{Tru.} & \multicolumn{1}{c}{Str.}\\ \midrule
{\sc HC(LR)} & .859 (.016) & .859 (.016) & .686 (.024) & .686 (.024) & .739 (.022) & .739 (.022) & .685 (.006) & .685 (.006)\\
{\sc HC(LR)} w/ disc. & .855 (.016) & .855 (.016) & .679 (.024) & .679 (.024) & .753 (.019) & .753 (.019) & .695 (.005) & .695 (.005)\\
{\sc HCapp(LR)} & .860 (.017) & .860 (.017) & .685 (.028) & .685 (.028) & .733 (.023) & .733 (.023) & .649 (.018) & .649 (.018)\\
{\sc HCapp(LR)} w/ disc. & .852 (.030) & .852 (.030) & .672 (.030) & .672 (.030) & .745 (.019) & .745 (.019) & .637 (.009) & .637 (.009)\\
{\sc HC(ANN)} & .859 (.015) & .859 (.015) & .691 (.023) & .691 (.023) & .729 (.025) & .729 (.025) & .688 (.006) & .688 (.006)\\
{\sc HC(ANN)} w/ disc. & .849 (.015) & .849 (.015) & .683 (.023) & .683 (.023) & .754 (.019) & .754 (.019) & .696 (.006) & .696 (.006)\\
\bottomrule
\end{tabular}
\label{tab:AAAI,accuracy,_0_bal,preFS}
\end{table*}

\begin{table*}\centering
\caption{Our methods vs. the rest: mean classifier accuracy for $\epsilon=0.1$, balanced datasets, pre-selected 4 features.}
\setlength{\tabcolsep}{2.7pt}\begin{tabular}{@{}lcccccccc@{}}\toprule
\multirow{2}{*}[-3pt]{\makecell[l]{Classifier}} & \multicolumn{2}{c}{Australia} & \multicolumn{2}{c}{Germany} & \multicolumn{2}{c}{Poland} & \multicolumn{2}{c}{Taiwan}\\
\cmidrule(l{7pt}r{7pt}){2-3} \cmidrule(l{7pt}r{7pt}){4-5} \cmidrule(l{7pt}r{7pt}){6-7} \cmidrule(l{7pt}r{7pt}){8-9}
& \multicolumn{1}{c}{Tru.} & \multicolumn{1}{c}{Str.} & \multicolumn{1}{c}{Tru.} & \multicolumn{1}{c}{Str.} & \multicolumn{1}{c}{Tru.} & \multicolumn{1}{c}{Str.} & \multicolumn{1}{c}{Tru.} & \multicolumn{1}{c}{Str.}\\ \midrule
{\sc HC(LR)} & .814 (.020) & .814 (.020) & .659 (.027) & .659 (.027) & .681 (.028) & .681 (.028) & .664 (.005) & .664 (.005)\\
{\sc HC(LR)} w/ disc. & .818 (.020) & .818 (.020) & .657 (.028) & .657 (.028) & .718 (.021) & .718 (.021) & .671 (.005) & .671 (.005)\\
{\sc HCapp(LR)} & .817 (.020) & .817 (.020) & .669 (.026) & .669 (.026) & .694 (.027) & .694 (.027) & .664 (.007) & .664 (.007)\\
{\sc HCapp(LR)} w/ disc. & .817 (.022) & .817 (.022) & .660 (.028) & .660 (.028) & .722 (.023) & .722 (.023) & .669 (.009) & .669 (.009)\\
{\sc HC(ANN)} & .803 (.024) & .803 (.024) & .647 (.029) & .647 (.029) & .664 (.028) & .664 (.028) & .665 (.006) & .665 (.006)\\
{\sc HC(ANN)} w/ disc. & .796 (.026) & .796 (.026) & .634 (.033) & .634 (.033) & .705 (.026) & .705 (.026) & .671 (.005) & .671 (.005)\\
\bottomrule
\end{tabular}
\label{tab:AAAI,accuracy,_1_bal,preFS}
\end{table*}

\begin{table*}\centering
\caption{Our methods vs. the rest: mean classifier accuracy for $\epsilon=0.2$, balanced datasets, pre-selected 4 features.}
\setlength{\tabcolsep}{2.7pt}\begin{tabular}{@{}lcccccccc@{}}\toprule
\multirow{2}{*}[-3pt]{\makecell[l]{Classifier}} & \multicolumn{2}{c}{Australia} & \multicolumn{2}{c}{Germany} & \multicolumn{2}{c}{Poland} & \multicolumn{2}{c}{Taiwan}\\
\cmidrule(l{7pt}r{7pt}){2-3} \cmidrule(l{7pt}r{7pt}){4-5} \cmidrule(l{7pt}r{7pt}){6-7} \cmidrule(l{7pt}r{7pt}){8-9}
& \multicolumn{1}{c}{Tru.} & \multicolumn{1}{c}{Str.} & \multicolumn{1}{c}{Tru.} & \multicolumn{1}{c}{Str.} & \multicolumn{1}{c}{Tru.} & \multicolumn{1}{c}{Str.} & \multicolumn{1}{c}{Tru.} & \multicolumn{1}{c}{Str.}\\ \midrule
{\sc HC(LR)} & .792 (.022) & .792 (.022) & .639 (.029) & .639 (.029) & .660 (.029) & .660 (.029) & .648 (.006) & .648 (.006)\\
{\sc HC(LR)} w/ disc. & .795 (.023) & .795 (.023) & .641 (.027) & .641 (.027) & .692 (.025) & .692 (.025) & .650 (.005) & .650 (.005)\\
{\sc HCapp(LR)} & .794 (.021) & .794 (.021) & .652 (.027) & .652 (.027) & .670 (.027) & .670 (.027) & .646 (.005) & .646 (.005)\\
{\sc HCapp(LR)} w/ disc. & .795 (.024) & .795 (.024) & .649 (.027) & .649 (.027) & .694 (.024) & .694 (.024) & .648 (.005) & .648 (.005)\\
{\sc HC(ANN)} & .775 (.028) & .775 (.028) & .624 (.029) & .624 (.029) & .640 (.035) & .640 (.035) & .647 (.006) & .647 (.006)\\
{\sc HC(ANN)} w/ disc. & .770 (.032) & .770 (.032) & .614 (.036) & .614 (.036) & .680 (.027) & .680 (.027) & .648 (.006) & .648 (.006)\\
\bottomrule
\end{tabular}
\label{tab:AAAI,accuracy,_2_bal,preFS}
\end{table*}

\begin{table*}\centering
\caption{Our methods vs. the rest: mean classifier accuracy for $\epsilon=0.3$, balanced datasets, pre-selected 4 features.}
\setlength{\tabcolsep}{2.7pt}\begin{tabular}{@{}lcccccccc@{}}\toprule
\multirow{2}{*}[-3pt]{\makecell[l]{Classifier}} & \multicolumn{2}{c}{Australia} & \multicolumn{2}{c}{Germany} & \multicolumn{2}{c}{Poland} & \multicolumn{2}{c}{Taiwan}\\
\cmidrule(l{7pt}r{7pt}){2-3} \cmidrule(l{7pt}r{7pt}){4-5} \cmidrule(l{7pt}r{7pt}){6-7} \cmidrule(l{7pt}r{7pt}){8-9}
& \multicolumn{1}{c}{Tru.} & \multicolumn{1}{c}{Str.} & \multicolumn{1}{c}{Tru.} & \multicolumn{1}{c}{Str.} & \multicolumn{1}{c}{Tru.} & \multicolumn{1}{c}{Str.} & \multicolumn{1}{c}{Tru.} & \multicolumn{1}{c}{Str.}\\ \midrule
{\sc HC(LR)} & .771 (.023) & .771 (.023) & .625 (.029) & .625 (.029) & .642 (.029) & .642 (.029) & .636 (.006) & .636 (.006)\\
{\sc HC(LR)} w/ disc. & .769 (.024) & .769 (.024) & .629 (.030) & .629 (.030) & .670 (.027) & .670 (.027) & .639 (.006) & .639 (.006)\\
{\sc HCapp(LR)} & .773 (.019) & .773 (.019) & .634 (.028) & .634 (.028) & .645 (.027) & .645 (.027) & .636 (.006) & .636 (.006)\\
{\sc HCapp(LR)} w/ disc. & .772 (.021) & .772 (.021) & .631 (.031) & .631 (.031) & .672 (.027) & .672 (.027) & .638 (.006) & .638 (.006)\\
{\sc HC(ANN)} & .753 (.028) & .753 (.028) & .610 (.030) & .610 (.030) & .621 (.034) & .621 (.034) & .636 (.005) & .636 (.005)\\
{\sc HC(ANN)} w/ disc. & .748 (.031) & .748 (.031) & .598 (.035) & .598 (.035) & .659 (.028) & .659 (.028) & .638 (.006) & .638 (.006)\\
\bottomrule
\end{tabular}
\label{tab:AAAI,accuracy,_3_bal,preFS}
\end{table*}

\begin{table*}\centering
\caption{Our methods vs. the rest: mean classifier accuracy for $\epsilon=0.4$, balanced datasets, pre-selected 4 features.}
\setlength{\tabcolsep}{2.7pt}\begin{tabular}{@{}lcccccccc@{}}\toprule
\multirow{2}{*}[-3pt]{\makecell[l]{Classifier}} & \multicolumn{2}{c}{Australia} & \multicolumn{2}{c}{Germany} & \multicolumn{2}{c}{Poland} & \multicolumn{2}{c}{Taiwan}\\
\cmidrule(l{7pt}r{7pt}){2-3} \cmidrule(l{7pt}r{7pt}){4-5} \cmidrule(l{7pt}r{7pt}){6-7} \cmidrule(l{7pt}r{7pt}){8-9}
& \multicolumn{1}{c}{Tru.} & \multicolumn{1}{c}{Str.} & \multicolumn{1}{c}{Tru.} & \multicolumn{1}{c}{Str.} & \multicolumn{1}{c}{Tru.} & \multicolumn{1}{c}{Str.} & \multicolumn{1}{c}{Tru.} & \multicolumn{1}{c}{Str.}\\ \midrule
{\sc HC(LR)} & .751 (.024) & .751 (.024) & .610 (.030) & .610 (.030) & .626 (.030) & .626 (.030) & .628 (.006) & .628 (.006)\\
{\sc HC(LR)} w/ disc. & .748 (.025) & .748 (.025) & .617 (.030) & .617 (.030) & .650 (.026) & .650 (.026) & .628 (.006) & .628 (.006)\\
{\sc HCapp(LR)} & .752 (.023) & .752 (.023) & .618 (.029) & .618 (.029) & .629 (.029) & .629 (.029) & .627 (.006) & .627 (.006)\\
{\sc HCapp(LR)} w/ disc. & .746 (.023) & .746 (.023) & .620 (.032) & .620 (.032) & .657 (.024) & {\bf .657} (.024) & .628 (.006) & .628 (.006)\\
{\sc HC(ANN)} & .731 (.028) & .731 (.028) & .593 (.031) & .593 (.031) & .601 (.035) & .601 (.035) & .628 (.007) & .628 (.007)\\
{\sc HC(ANN)} w/ disc. & .726 (.030) & .726 (.030) & .586 (.037) & .586 (.037) & .640 (.029) & .640 (.029) & .627 (.006) & .627 (.006)\\
\bottomrule
\end{tabular}
\label{tab:AAAI,accuracy,_4_bal,preFS}
\end{table*}

\begin{table*}\centering
\caption{Our methods vs. the rest: mean classifier accuracy for $\epsilon=0.5$, balanced datasets, pre-selected 4 features.}
\setlength{\tabcolsep}{2.7pt}\begin{tabular}{@{}lcccccccc@{}}\toprule
\multirow{2}{*}[-3pt]{\makecell[l]{Classifier}} & \multicolumn{2}{c}{Australia} & \multicolumn{2}{c}{Germany} & \multicolumn{2}{c}{Poland} & \multicolumn{2}{c}{Taiwan}\\
\cmidrule(l{7pt}r{7pt}){2-3} \cmidrule(l{7pt}r{7pt}){4-5} \cmidrule(l{7pt}r{7pt}){6-7} \cmidrule(l{7pt}r{7pt}){8-9}
& \multicolumn{1}{c}{Tru.} & \multicolumn{1}{c}{Str.} & \multicolumn{1}{c}{Tru.} & \multicolumn{1}{c}{Str.} & \multicolumn{1}{c}{Tru.} & \multicolumn{1}{c}{Str.} & \multicolumn{1}{c}{Tru.} & \multicolumn{1}{c}{Str.}\\ \midrule
{\sc HC(LR)} & .726 (.025) & .726 (.025) & .596 (.029) & .596 (.029) & .612 (.030) & .612 (.030) & .616 (.006) & .616 (.006)\\
{\sc HC(LR)} w/ disc. & .719 (.025) & .719 (.025) & .601 (.030) & .601 (.030) & .638 (.028) & {\bf .638} (.028) & .616 (.005) & .616 (.005)\\
{\sc HCapp(LR)} & .727 (.026) & {\bf .727} (.026) & .600 (.028) & .600 (.028) & .617 (.029) & .617 (.029) & .616 (.007) & .616 (.007)\\
{\sc HCapp(LR)} w/ disc. & .720 (.027) & .720 (.027) & .603 (.032) & {\bf .603} (.032) & .636 (.026) & .636 (.026) & .617 (.005) & .617 (.005)\\
{\sc HC(ANN)} & .709 (.028) & .709 (.028) & .578 (.032) & .578 (.032) & .588 (.036) & .588 (.036) & .615 (.006) & .615 (.006)\\
{\sc HC(ANN)} w/ disc. & .701 (.032) & .701 (.032) & .570 (.036) & .570 (.036) & .626 (.030) & .626 (.030) & .615 (.006) & .615 (.006)\\
\bottomrule
\end{tabular}
\label{tab:AAAI,accuracy,_5_bal,preFS}
\end{table*}

\begin{table*}\centering
\caption{Our methods vs. the rest: mean classifier AUC for $\epsilon=0.0$, balanced datasets, pre-selected 4 features.}
\setlength{\tabcolsep}{2.7pt}\begin{tabular}{@{}lcccccccc@{}}\toprule
\multirow{2}{*}[-3pt]{\makecell[l]{Classifier}} & \multicolumn{2}{c}{Australia} & \multicolumn{2}{c}{Germany} & \multicolumn{2}{c}{Poland} & \multicolumn{2}{c}{Taiwan}\\
\cmidrule(l{7pt}r{7pt}){2-3} \cmidrule(l{7pt}r{7pt}){4-5} \cmidrule(l{7pt}r{7pt}){6-7} \cmidrule(l{7pt}r{7pt}){8-9}
& \multicolumn{1}{c}{Tru.} & \multicolumn{1}{c}{Str.} & \multicolumn{1}{c}{Tru.} & \multicolumn{1}{c}{Str.} & \multicolumn{1}{c}{Tru.} & \multicolumn{1}{c}{Str.} & \multicolumn{1}{c}{Tru.} & \multicolumn{1}{c}{Str.}\\ \midrule
{\sc HC(LR)} & .906 (.013) & .906 (.013) & .739 (.026) & .739 (.026) & .807 (.022) & .807 (.022) & .696 (.007) & .696 (.007)\\
{\sc HC(LR)} w/ disc. & {\bf .911} (.013) & {\bf .911} (.013) & .732 (.024) & .732 (.024) & .813 (.019) & .813 (.019) & .729 (.006) & .729 (.006)\\
{\sc HCapp(LR)} & .910 (.013) & .910 (.013) & .732 (.027) & .732 (.027) & .795 (.023) & .795 (.023) & .649 (.014) & .649 (.014)\\
{\sc HCapp(LR)} w/ disc. & {\bf .911} (.019) & {\bf .911} (.019) & .720 (.027) & .720 (.027) & .786 (.021) & .786 (.021) & .646 (.010) & .646 (.010)\\
{\sc HC(ANN)} & .908 (.013) & .908 (.013) & .753 (.023) & .753 (.023) & .800 (.023) & .800 (.023) & .701 (.010) & .701 (.010)\\
{\sc HC(ANN)} w/ disc. & .909 (.013) & .909 (.013) & .736 (.024) & .736 (.024) & .815 (.018) & .815 (.018) & .731 (.005) & .731 (.005)\\
\bottomrule
\end{tabular}
\label{tab:AAAI,AUC,_0_bal,preFS}
\end{table*}

\begin{table*}\centering
\caption{Our methods vs. the rest: mean classifier AUC for $\epsilon=0.1$, balanced datasets, pre-selected 4 features.}
\setlength{\tabcolsep}{2.7pt}\begin{tabular}{@{}lcccccccc@{}}\toprule
\multirow{2}{*}[-3pt]{\makecell[l]{Classifier}} & \multicolumn{2}{c}{Australia} & \multicolumn{2}{c}{Germany} & \multicolumn{2}{c}{Poland} & \multicolumn{2}{c}{Taiwan}\\
\cmidrule(l{7pt}r{7pt}){2-3} \cmidrule(l{7pt}r{7pt}){4-5} \cmidrule(l{7pt}r{7pt}){6-7} \cmidrule(l{7pt}r{7pt}){8-9}
& \multicolumn{1}{c}{Tru.} & \multicolumn{1}{c}{Str.} & \multicolumn{1}{c}{Tru.} & \multicolumn{1}{c}{Str.} & \multicolumn{1}{c}{Tru.} & \multicolumn{1}{c}{Str.} & \multicolumn{1}{c}{Tru.} & \multicolumn{1}{c}{Str.}\\ \midrule
{\sc HC(LR)} & .889 (.018) & .889 (.018) & .704 (.028) & .704 (.028) & .751 (.029) & .751 (.029) & .674 (.006) & .674 (.006)\\
{\sc HC(LR)} w/ disc. & .888 (.018) & .888 (.018) & .702 (.027) & .702 (.027) & .774 (.022) & .774 (.022) & .704 (.006) & .704 (.006)\\
{\sc HCapp(LR)} & .893 (.016) & {\bf .893} (.016) & .713 (.030) & .713 (.030) & .753 (.028) & .753 (.028) & .674 (.007) & .674 (.007)\\
{\sc HCapp(LR)} w/ disc. & .891 (.017) & .891 (.017) & .703 (.028) & .703 (.028) & .771 (.023) & .771 (.023) & .703 (.006) & .703 (.006)\\
{\sc HC(ANN)} & .886 (.018) & .886 (.018) & .707 (.028) & .707 (.028) & .730 (.032) & .730 (.032) & .682 (.008) & .682 (.008)\\
{\sc HC(ANN)} w/ disc. & .876 (.019) & .876 (.019) & .681 (.032) & .681 (.032) & .765 (.023) & .765 (.023) & .706 (.006) & .706 (.006)\\
\bottomrule
\end{tabular}
\label{tab:AAAI,AUC,_1_bal,preFS}
\end{table*}

\begin{table*}\centering
\caption{Our methods vs. the rest: mean classifier AUC for $\epsilon=0.2$, balanced datasets, pre-selected 4 features.}
\setlength{\tabcolsep}{2.7pt}\begin{tabular}{@{}lcccccccc@{}}\toprule
\multirow{2}{*}[-3pt]{\makecell[l]{Classifier}} & \multicolumn{2}{c}{Australia} & \multicolumn{2}{c}{Germany} & \multicolumn{2}{c}{Poland} & \multicolumn{2}{c}{Taiwan}\\
\cmidrule(l{7pt}r{7pt}){2-3} \cmidrule(l{7pt}r{7pt}){4-5} \cmidrule(l{7pt}r{7pt}){6-7} \cmidrule(l{7pt}r{7pt}){8-9}
& \multicolumn{1}{c}{Tru.} & \multicolumn{1}{c}{Str.} & \multicolumn{1}{c}{Tru.} & \multicolumn{1}{c}{Str.} & \multicolumn{1}{c}{Tru.} & \multicolumn{1}{c}{Str.} & \multicolumn{1}{c}{Tru.} & \multicolumn{1}{c}{Str.}\\ \midrule
{\sc HC(LR)} & .875 (.018) & .875 (.018) & .678 (.031) & .678 (.031) & .727 (.028) & .727 (.028) & .662 (.007) & .662 (.007)\\
{\sc HC(LR)} w/ disc. & .870 (.019) & .870 (.019) & .679 (.030) & .679 (.030) & .748 (.025) & .748 (.025) & .693 (.007) & .693 (.007)\\
{\sc HCapp(LR)} & .876 (.018) & .876 (.018) & .691 (.029) & .691 (.029) & .731 (.027) & .731 (.027) & .661 (.007) & .661 (.007)\\
{\sc HCapp(LR)} w/ disc. & .871 (.020) & .871 (.020) & .684 (.028) & .684 (.028) & .746 (.024) & .746 (.024) & .691 (.006) & .691 (.006)\\
{\sc HC(ANN)} & .865 (.021) & .865 (.021) & .675 (.032) & .675 (.032) & .702 (.036) & .702 (.036) & .670 (.008) & .670 (.008)\\
{\sc HC(ANN)} w/ disc. & .850 (.026) & .850 (.026) & .655 (.035) & .655 (.035) & .735 (.027) & .735 (.027) & .691 (.007) & .691 (.007)\\
\bottomrule
\end{tabular}
\label{tab:AAAI,AUC,_2_bal,preFS}
\end{table*}

\begin{table*}\centering
\caption{Our methods vs. the rest: mean classifier AUC for $\epsilon=0.3$, balanced datasets, pre-selected 4 features.}
\setlength{\tabcolsep}{2.7pt}\begin{tabular}{@{}lcccccccc@{}}\toprule
\multirow{2}{*}[-3pt]{\makecell[l]{Classifier}} & \multicolumn{2}{c}{Australia} & \multicolumn{2}{c}{Germany} & \multicolumn{2}{c}{Poland} & \multicolumn{2}{c}{Taiwan}\\
\cmidrule(l{7pt}r{7pt}){2-3} \cmidrule(l{7pt}r{7pt}){4-5} \cmidrule(l{7pt}r{7pt}){6-7} \cmidrule(l{7pt}r{7pt}){8-9}
& \multicolumn{1}{c}{Tru.} & \multicolumn{1}{c}{Str.} & \multicolumn{1}{c}{Tru.} & \multicolumn{1}{c}{Str.} & \multicolumn{1}{c}{Tru.} & \multicolumn{1}{c}{Str.} & \multicolumn{1}{c}{Tru.} & \multicolumn{1}{c}{Str.}\\ \midrule
{\sc HC(LR)} & .855 (.020) & .855 (.020) & .658 (.032) & .658 (.032) & .704 (.027) & .704 (.027) & .653 (.006) & .653 (.006)\\
{\sc HC(LR)} w/ disc. & .848 (.021) & .848 (.021) & .659 (.032) & .659 (.032) & .724 (.029) & .724 (.029) & .676 (.008) & .676 (.008)\\
{\sc HCapp(LR)} & .854 (.018) & .854 (.018) & .667 (.031) & .667 (.031) & .702 (.030) & .702 (.030) & .653 (.006) & .653 (.006)\\
{\sc HCapp(LR)} w/ disc. & .847 (.018) & .847 (.018) & .661 (.030) & .661 (.030) & .721 (.027) & .721 (.027) & .677 (.006) & .677 (.006)\\
{\sc HC(ANN)} & .843 (.023) & .843 (.023) & .651 (.032) & .651 (.032) & .674 (.039) & .674 (.039) & .656 (.006) & .656 (.006)\\
{\sc HC(ANN)} w/ disc. & .826 (.031) & .826 (.031) & .628 (.037) & .628 (.037) & .709 (.030) & .709 (.030) & .679 (.006) & .679 (.006)\\
\bottomrule
\end{tabular}
\label{tab:AAAI,AUC,_3_bal,preFS}
\end{table*}

\begin{table*}\centering
\caption{Our methods vs. the rest: mean classifier AUC for $\epsilon=0.4$, balanced datasets, pre-selected 4 features.}
\setlength{\tabcolsep}{2.7pt}\begin{tabular}{@{}lcccccccc@{}}\toprule
\multirow{2}{*}[-3pt]{\makecell[l]{Classifier}} & \multicolumn{2}{c}{Australia} & \multicolumn{2}{c}{Germany} & \multicolumn{2}{c}{Poland} & \multicolumn{2}{c}{Taiwan}\\
\cmidrule(l{7pt}r{7pt}){2-3} \cmidrule(l{7pt}r{7pt}){4-5} \cmidrule(l{7pt}r{7pt}){6-7} \cmidrule(l{7pt}r{7pt}){8-9}
& \multicolumn{1}{c}{Tru.} & \multicolumn{1}{c}{Str.} & \multicolumn{1}{c}{Tru.} & \multicolumn{1}{c}{Str.} & \multicolumn{1}{c}{Tru.} & \multicolumn{1}{c}{Str.} & \multicolumn{1}{c}{Tru.} & \multicolumn{1}{c}{Str.}\\ \midrule
{\sc HC(LR)} & .828 (.023) & .828 (.023) & .639 (.034) & .639 (.034) & .681 (.030) & .681 (.030) & .644 (.007) & .644 (.007)\\
{\sc HC(LR)} w/ disc. & .820 (.024) & .820 (.024) & .641 (.032) & .641 (.032) & .695 (.027) & .695 (.027) & .663 (.006) & .663 (.006)\\
{\sc HCapp(LR)} & .827 (.021) & .827 (.021) & .645 (.033) & .645 (.033) & .682 (.030) & .682 (.030) & .644 (.006) & .644 (.006)\\
{\sc HCapp(LR)} w/ disc. & .817 (.022) & .817 (.022) & .642 (.032) & .642 (.032) & .699 (.026) & .699 (.026) & .663 (.006) & .663 (.006)\\
{\sc HC(ANN)} & .810 (.028) & .810 (.028) & .626 (.037) & .626 (.037) & .646 (.040) & .646 (.040) & .647 (.007) & .647 (.007)\\
{\sc HC(ANN)} w/ disc. & .794 (.034) & .794 (.034) & .604 (.037) & .604 (.037) & .683 (.028) & .683 (.028) & .664 (.007) & .664 (.007)\\
\bottomrule
\end{tabular}
\label{tab:AAAI,AUC,_4_bal,preFS}
\end{table*}

\begin{table*}\centering
\caption{Our methods vs. the rest: mean classifier AUC for $\epsilon=0.5$, balanced datasets, pre-selected 4 features.}
\setlength{\tabcolsep}{2.7pt}\begin{tabular}{@{}lcccccccc@{}}\toprule
\multirow{2}{*}[-3pt]{\makecell[l]{Classifier}} & \multicolumn{2}{c}{Australia} & \multicolumn{2}{c}{Germany} & \multicolumn{2}{c}{Poland} & \multicolumn{2}{c}{Taiwan}\\
\cmidrule(l{7pt}r{7pt}){2-3} \cmidrule(l{7pt}r{7pt}){4-5} \cmidrule(l{7pt}r{7pt}){6-7} \cmidrule(l{7pt}r{7pt}){8-9}
& \multicolumn{1}{c}{Tru.} & \multicolumn{1}{c}{Str.} & \multicolumn{1}{c}{Tru.} & \multicolumn{1}{c}{Str.} & \multicolumn{1}{c}{Tru.} & \multicolumn{1}{c}{Str.} & \multicolumn{1}{c}{Tru.} & \multicolumn{1}{c}{Str.}\\ \midrule
{\sc HC(LR)} & .788 (.023) & .788 (.023) & .616 (.033) & .616 (.033) & .657 (.030) & .657 (.030) & .630 (.006) & .630 (.006)\\
{\sc HC(LR)} w/ disc. & .779 (.024) & .779 (.024) & .618 (.031) & .618 (.031) & .670 (.029) & .670 (.029) & .645 (.006) & .645 (.006)\\
{\sc HCapp(LR)} & .787 (.025) & .787 (.025) & .620 (.032) & .620 (.032) & .661 (.031) & .661 (.031) & .631 (.007) & .631 (.007)\\
{\sc HCapp(LR)} w/ disc. & .777 (.027) & .777 (.027) & .620 (.031) & .620 (.031) & .670 (.027) & .670 (.027) & .646 (.006) & .646 (.006)\\
{\sc HC(ANN)} & .773 (.029) & .773 (.029) & .600 (.035) & .600 (.035) & .624 (.041) & .624 (.041) & .632 (.006) & .632 (.006)\\
{\sc HC(ANN)} w/ disc. & .761 (.031) & .761 (.031) & .581 (.036) & .581 (.036) & .661 (.032) & .661 (.032) & .646 (.006) & .646 (.006)\\
\bottomrule
\end{tabular}
\label{tab:AAAI,AUC,_5_bal,preFS}
\end{table*}

\begin{table*}\centering
\caption{Our methods vs. the rest: mean classifier accuracy for $\epsilon=0.0$, unbalanced datasets, pre-selected 4 features.}
\setlength{\tabcolsep}{2.7pt}\begin{tabular}{@{}lcccccccc@{}}\toprule
\multirow{2}{*}[-3pt]{\makecell[l]{Classifier}} & \multicolumn{2}{c}{Australia} & \multicolumn{2}{c}{Germany} & \multicolumn{2}{c}{Poland} & \multicolumn{2}{c}{Taiwan}\\
\cmidrule(l{7pt}r{7pt}){2-3} \cmidrule(l{7pt}r{7pt}){4-5} \cmidrule(l{7pt}r{7pt}){6-7} \cmidrule(l{7pt}r{7pt}){8-9}
& \multicolumn{1}{c}{Tru.} & \multicolumn{1}{c}{Str.} & \multicolumn{1}{c}{Tru.} & \multicolumn{1}{c}{Str.} & \multicolumn{1}{c}{Tru.} & \multicolumn{1}{c}{Str.} & \multicolumn{1}{c}{Tru.} & \multicolumn{1}{c}{Str.}\\ \midrule
{\sc HC(LR)} & .848 (.015) & .848 (.015) & .736 (.017) & .736 (.017) & .931 (.003) & .931 (.003) & .808 (.003) & .808 (.003)\\
{\sc HC(LR)} w/ disc. & .848 (.013) & .848 (.013) & .730 (.015) & .730 (.015) & .931 (.004) & .931 (.004) & .819 (.002) & .819 (.002)\\
{\sc HCapp(LR)} & .849 (.014) & .849 (.014) & .720 (.017) & .720 (.017) & .931 (.003) & .931 (.003) & .808 (.013) & .808 (.013)\\
{\sc HCapp(LR)} w/ disc. & .834 (.036) & .834 (.036) & .718 (.015) & .718 (.015) & {\bf .932} (.003) & {\bf .932} (.003) & .819 (.002) & .819 (.002)\\
{\sc HC(ANN)} & .850 (.013) & .850 (.013) & .740 (.018) & .740 (.018) & .931 (.003) & .931 (.003) & .819 (.002) & .819 (.002)\\
{\sc HC(ANN)} w/ disc. & .848 (.014) & .848 (.014) & .730 (.017) & .730 (.017) & .931 (.003) & .931 (.003) & .819 (.002) & .819 (.002)\\
\bottomrule
\end{tabular}
\label{tab:AAAI,accuracy,_0,preFS}
\end{table*}

\begin{table*}\centering
\caption{Our methods vs. the rest: mean classifier accuracy for $\epsilon=0.1$, unbalanced datasets, pre-selected 4 features.}
\setlength{\tabcolsep}{2.7pt}\begin{tabular}{@{}lcccccccc@{}}\toprule
\multirow{2}{*}[-3pt]{\makecell[l]{Classifier}} & \multicolumn{2}{c}{Australia} & \multicolumn{2}{c}{Germany} & \multicolumn{2}{c}{Poland} & \multicolumn{2}{c}{Taiwan}\\
\cmidrule(l{7pt}r{7pt}){2-3} \cmidrule(l{7pt}r{7pt}){4-5} \cmidrule(l{7pt}r{7pt}){6-7} \cmidrule(l{7pt}r{7pt}){8-9}
& \multicolumn{1}{c}{Tru.} & \multicolumn{1}{c}{Str.} & \multicolumn{1}{c}{Tru.} & \multicolumn{1}{c}{Str.} & \multicolumn{1}{c}{Tru.} & \multicolumn{1}{c}{Str.} & \multicolumn{1}{c}{Tru.} & \multicolumn{1}{c}{Str.}\\ \midrule
{\sc HC(LR)} & .814 (.019) & .814 (.019) & .708 (.017) & .708 (.017) & .930 (.004) & .930 (.004) & .803 (.003) & .803 (.003)\\
{\sc HC(LR)} w/ disc. & .818 (.020) & .818 (.020) & .701 (.018) & .701 (.018) & .929 (.004) & .929 (.004) & .804 (.003) & {\bf .804} (.003)\\
{\sc HCapp(LR)} & .807 (.014) & .807 (.014) & .710 (.013) & .710 (.013) & .930 (.003) & .930 (.003) & .795 (.010) & .795 (.010)\\
{\sc HCapp(LR)} w/ disc. & .815 (.020) & .815 (.020) & .702 (.018) & .702 (.018) & .930 (.003) & .930 (.003) & .804 (.003) & {\bf .804} (.003)\\
{\sc HC(ANN)} & .799 (.022) & .799 (.022) & .700 (.020) & .700 (.020) & .930 (.004) & .930 (.004) & .803 (.003) & .803 (.003)\\
{\sc HC(ANN)} w/ disc. & .797 (.025) & .797 (.025) & .693 (.017) & .693 (.017) & .930 (.004) & .930 (.004) & .802 (.003) & .802 (.003)\\
\bottomrule
\end{tabular}
\label{tab:AAAI,accuracy,_1,preFS}
\end{table*}

\begin{table*}\centering
\caption{Our methods vs. the rest: mean classifier accuracy for $\epsilon=0.2$, unbalanced datasets, pre-selected 4 features.}
\setlength{\tabcolsep}{2.7pt}\begin{tabular}{@{}lcccccccc@{}}\toprule
\multirow{2}{*}[-3pt]{\makecell[l]{Classifier}} & \multicolumn{2}{c}{Australia} & \multicolumn{2}{c}{Germany} & \multicolumn{2}{c}{Poland} & \multicolumn{2}{c}{Taiwan}\\
\cmidrule(l{7pt}r{7pt}){2-3} \cmidrule(l{7pt}r{7pt}){4-5} \cmidrule(l{7pt}r{7pt}){6-7} \cmidrule(l{7pt}r{7pt}){8-9}
& \multicolumn{1}{c}{Tru.} & \multicolumn{1}{c}{Str.} & \multicolumn{1}{c}{Tru.} & \multicolumn{1}{c}{Str.} & \multicolumn{1}{c}{Tru.} & \multicolumn{1}{c}{Str.} & \multicolumn{1}{c}{Tru.} & \multicolumn{1}{c}{Str.}\\ \midrule
{\sc HC(LR)} & .796 (.019) & .796 (.019) & .694 (.019) & .694 (.019) & .929 (.003) & .929 (.003) & .800 (.002) & {\bf .800} (.002)\\
{\sc HC(LR)} w/ disc. & .794 (.020) & .794 (.020) & .688 (.020) & .688 (.020) & .929 (.003) & .929 (.003) & .800 (.002) & {\bf .800} (.002)\\
{\sc HCapp(LR)} & .795 (.019) & .795 (.019) & .692 (.016) & .692 (.016) & .929 (.004) & .929 (.004) & .789 (.011) & .789 (.011)\\
{\sc HCapp(LR)} w/ disc. & .789 (.022) & .789 (.022) & .690 (.019) & .690 (.019) & .928 (.003) & .928 (.003) & .800 (.002) & {\bf .800} (.002)\\
{\sc HC(ANN)} & .780 (.026) & .780 (.026) & .682 (.021) & .682 (.021) & .929 (.003) & .929 (.003) & .800 (.003) & {\bf .800} (.003)\\
{\sc HC(ANN)} w/ disc. & .771 (.027) & .771 (.027) & .686 (.021) & .686 (.021) & .929 (.003) & .929 (.003) & .800 (.002) & {\bf .800} (.002)\\
\bottomrule
\end{tabular}
\label{tab:AAAI,accuracy,_2,preFS}
\end{table*}

\begin{table*}\centering
\caption{Our methods vs. the rest: mean classifier accuracy for $\epsilon=0.3$, unbalanced datasets, pre-selected 4 features.}
\setlength{\tabcolsep}{2.7pt}\begin{tabular}{@{}lcccccccc@{}}\toprule
\multirow{2}{*}[-3pt]{\makecell[l]{Classifier}} & \multicolumn{2}{c}{Australia} & \multicolumn{2}{c}{Germany} & \multicolumn{2}{c}{Poland} & \multicolumn{2}{c}{Taiwan}\\
\cmidrule(l{7pt}r{7pt}){2-3} \cmidrule(l{7pt}r{7pt}){4-5} \cmidrule(l{7pt}r{7pt}){6-7} \cmidrule(l{7pt}r{7pt}){8-9}
& \multicolumn{1}{c}{Tru.} & \multicolumn{1}{c}{Str.} & \multicolumn{1}{c}{Tru.} & \multicolumn{1}{c}{Str.} & \multicolumn{1}{c}{Tru.} & \multicolumn{1}{c}{Str.} & \multicolumn{1}{c}{Tru.} & \multicolumn{1}{c}{Str.}\\ \midrule
{\sc HC(LR)} & .775 (.021) & .775 (.021) & .685 (.019) & .685 (.019) & .923 (.004) & .923 (.004) & .796 (.003) & .796 (.003)\\
{\sc HC(LR)} w/ disc. & .771 (.023) & .771 (.023) & .682 (.021) & .682 (.021) & .923 (.004) & .923 (.004) & .796 (.003) & .796 (.003)\\
{\sc HCapp(LR)} & .776 (.021) & {\bf .776} (.021) & .690 (.017) & .690 (.017) & .923 (.003) & .923 (.003) & .789 (.010) & .789 (.010)\\
{\sc HCapp(LR)} w/ disc. & .771 (.026) & .771 (.026) & .694 (.018) & .694 (.018) & .922 (.003) & .922 (.003) & .796 (.003) & .796 (.003)\\
{\sc HC(ANN)} & .755 (.026) & .755 (.026) & .673 (.022) & .673 (.022) & .923 (.004) & .923 (.004) & .796 (.003) & .796 (.003)\\
{\sc HC(ANN)} w/ disc. & .747 (.031) & .747 (.031) & .692 (.020) & .692 (.020) & .923 (.004) & .923 (.004) & .796 (.003) & .796 (.003)\\
\bottomrule
\end{tabular}
\label{tab:AAAI,accuracy,_3,preFS}
\end{table*}

\begin{table*}\centering
\caption{Our methods vs. the rest: mean classifier accuracy for $\epsilon=0.4$, unbalanced datasets, pre-selected 4 features.}
\setlength{\tabcolsep}{2.7pt}\begin{tabular}{@{}lcccccccc@{}}\toprule
\multirow{2}{*}[-3pt]{\makecell[l]{Classifier}} & \multicolumn{2}{c}{Australia} & \multicolumn{2}{c}{Germany} & \multicolumn{2}{c}{Poland} & \multicolumn{2}{c}{Taiwan}\\
\cmidrule(l{7pt}r{7pt}){2-3} \cmidrule(l{7pt}r{7pt}){4-5} \cmidrule(l{7pt}r{7pt}){6-7} \cmidrule(l{7pt}r{7pt}){8-9}
& \multicolumn{1}{c}{Tru.} & \multicolumn{1}{c}{Str.} & \multicolumn{1}{c}{Tru.} & \multicolumn{1}{c}{Str.} & \multicolumn{1}{c}{Tru.} & \multicolumn{1}{c}{Str.} & \multicolumn{1}{c}{Tru.} & \multicolumn{1}{c}{Str.}\\ \midrule
{\sc HC(LR)} & .756 (.023) & .756 (.023) & .673 (.021) & .673 (.021) & .908 (.005) & .908 (.005) & .786 (.002) & .786 (.002)\\
{\sc HC(LR)} w/ disc. & .749 (.025) & .749 (.025) & .686 (.023) & .686 (.023) & .908 (.005) & .908 (.005) & .786 (.002) & .786 (.002)\\
{\sc HCapp(LR)} & .759 (.023) & {\bf .759} (.023) & .675 (.018) & .675 (.018) & .908 (.004) & .908 (.004) & .779 (.010) & .779 (.010)\\
{\sc HCapp(LR)} w/ disc. & .749 (.025) & .749 (.025) & .692 (.017) & .692 (.017) & .908 (.004) & .908 (.004) & .786 (.002) & .786 (.002)\\
{\sc HC(ANN)} & .733 (.027) & .733 (.027) & .669 (.022) & .669 (.022) & .908 (.004) & .908 (.004) & .786 (.002) & .786 (.002)\\
{\sc HC(ANN)} w/ disc. & .724 (.033) & .724 (.033) & .694 (.019) & .694 (.019) & .909 (.005) & .909 (.005) & .786 (.003) & .786 (.003)\\
\bottomrule
\end{tabular}
\label{tab:AAAI,accuracy,_4,preFS}
\end{table*}

\begin{table*}\centering
\caption{Our methods vs. the rest: mean classifier accuracy for $\epsilon=0.5$, unbalanced datasets, pre-selected 4 features.}
\setlength{\tabcolsep}{2.7pt}\begin{tabular}{@{}lcccccccc@{}}\toprule
\multirow{2}{*}[-3pt]{\makecell[l]{Classifier}} & \multicolumn{2}{c}{Australia} & \multicolumn{2}{c}{Germany} & \multicolumn{2}{c}{Poland} & \multicolumn{2}{c}{Taiwan}\\
\cmidrule(l{7pt}r{7pt}){2-3} \cmidrule(l{7pt}r{7pt}){4-5} \cmidrule(l{7pt}r{7pt}){6-7} \cmidrule(l{7pt}r{7pt}){8-9}
& \multicolumn{1}{c}{Tru.} & \multicolumn{1}{c}{Str.} & \multicolumn{1}{c}{Tru.} & \multicolumn{1}{c}{Str.} & \multicolumn{1}{c}{Tru.} & \multicolumn{1}{c}{Str.} & \multicolumn{1}{c}{Tru.} & \multicolumn{1}{c}{Str.}\\ \midrule
{\sc HC(LR)} & .729 (.026) & .729 (.026) & .660 (.021) & .660 (.021) & .877 (.005) & .877 (.005) & .764 (.003) & .764 (.003)\\
{\sc HC(LR)} w/ disc. & .724 (.025) & .724 (.025) & .692 (.023) & .692 (.023) & .880 (.013) & .880 (.013) & .765 (.003) & .765 (.003)\\
{\sc HCapp(LR)} & .734 (.021) & {\bf .734} (.021) & .666 (.021) & .666 (.021) & .877 (.005) & .877 (.005) & .759 (.008) & .759 (.008)\\
{\sc HCapp(LR)} w/ disc. & .730 (.023) & .730 (.023) & .692 (.023) & .692 (.023) & .881 (.014) & .881 (.014) & .764 (.002) & .764 (.002)\\
{\sc HC(ANN)} & .714 (.030) & .714 (.030) & .663 (.021) & .663 (.021) & .877 (.005) & .877 (.005) & .765 (.003) & .765 (.003)\\
{\sc HC(ANN)} w/ disc. & .706 (.032) & .706 (.032) & .694 (.019) & .694 (.019) & .880 (.013) & .880 (.013) & .765 (.003) & .765 (.003)\\
\bottomrule
\end{tabular}
\label{tab:AAAI,accuracy,_5,preFS}
\end{table*}

\begin{table*}\centering
\caption{Our methods vs. the rest: mean classifier AUC for $\epsilon=0.0$, unbalanced datasets, pre-selected 4 features.}
\setlength{\tabcolsep}{2.7pt}\begin{tabular}{@{}lcccccccc@{}}\toprule
\multirow{2}{*}[-3pt]{\makecell[l]{Classifier}} & \multicolumn{2}{c}{Australia} & \multicolumn{2}{c}{Germany} & \multicolumn{2}{c}{Poland} & \multicolumn{2}{c}{Taiwan}\\
\cmidrule(l{7pt}r{7pt}){2-3} \cmidrule(l{7pt}r{7pt}){4-5} \cmidrule(l{7pt}r{7pt}){6-7} \cmidrule(l{7pt}r{7pt}){8-9}
& \multicolumn{1}{c}{Tru.} & \multicolumn{1}{c}{Str.} & \multicolumn{1}{c}{Tru.} & \multicolumn{1}{c}{Str.} & \multicolumn{1}{c}{Tru.} & \multicolumn{1}{c}{Str.} & \multicolumn{1}{c}{Tru.} & \multicolumn{1}{c}{Str.}\\ \midrule
{\sc HC(LR)} & .907 (.012) & .907 (.012) & .723 (.033) & .723 (.033) & .729 (.087) & .729 (.087) & .705 (.004) & .705 (.004)\\
{\sc HC(LR)} w/ disc. & {\bf .911} (.012) & {\bf .911} (.012) & .738 (.020) & .738 (.020) & .797 (.048) & .797 (.048) & .730 (.007) & .730 (.007)\\
{\sc HCapp(LR)} & .910 (.012) & .910 (.012) & .694 (.034) & .694 (.034) & .729 (.067) & .729 (.067) & .674 (.026) & .674 (.026)\\
{\sc HCapp(LR)} w/ disc. & .904 (.030) & .904 (.030) & .685 (.031) & .685 (.031) & .781 (.019) & .781 (.019) & .717 (.004) & .717 (.004)\\
{\sc HC(ANN)} & .908 (.012) & .908 (.012) & .745 (.026) & .745 (.026) & .795 (.050) & .795 (.050) & .729 (.014) & .729 (.014)\\
{\sc HC(ANN)} w/ disc. & .910 (.012) & .910 (.012) & .747 (.019) & .747 (.019) & .818 (.040) & .818 (.040) & .730 (.007) & .730 (.007)\\
\bottomrule
\end{tabular}
\label{tab:AAAI,AUC,_0,preFS}
\end{table*}

\begin{table*}\centering
\caption{Our methods vs. the rest: mean classifier AUC for $\epsilon=0.1$, unbalanced datasets, pre-selected 4 features.}
\setlength{\tabcolsep}{2.7pt}\begin{tabular}{@{}lcccccccc@{}}\toprule
\multirow{2}{*}[-3pt]{\makecell[l]{Classifier}} & \multicolumn{2}{c}{Australia} & \multicolumn{2}{c}{Germany} & \multicolumn{2}{c}{Poland} & \multicolumn{2}{c}{Taiwan}\\
\cmidrule(l{7pt}r{7pt}){2-3} \cmidrule(l{7pt}r{7pt}){4-5} \cmidrule(l{7pt}r{7pt}){6-7} \cmidrule(l{7pt}r{7pt}){8-9}
& \multicolumn{1}{c}{Tru.} & \multicolumn{1}{c}{Str.} & \multicolumn{1}{c}{Tru.} & \multicolumn{1}{c}{Str.} & \multicolumn{1}{c}{Tru.} & \multicolumn{1}{c}{Str.} & \multicolumn{1}{c}{Tru.} & \multicolumn{1}{c}{Str.}\\ \midrule
{\sc HC(LR)} & .890 (.018) & .890 (.018) & .682 (.033) & .682 (.033) & .653 (.072) & .653 (.072) & .655 (.013) & .655 (.013)\\
{\sc HC(LR)} w/ disc. & .890 (.016) & .890 (.016) & .687 (.031) & .687 (.031) & .715 (.041) & .715 (.041) & .681 (.020) & .681 (.020)\\
{\sc HCapp(LR)} & .887 (.014) & .887 (.014) & .684 (.039) & .684 (.039) & .646 (.081) & .646 (.081) & .637 (.025) & .637 (.025)\\
{\sc HCapp(LR)} w/ disc. & .886 (.014) & .886 (.014) & .665 (.039) & .665 (.039) & .685 (.093) & .685 (.093) & .682 (.016) & .682 (.016)\\
{\sc HC(ANN)} & .886 (.016) & .886 (.016) & .690 (.031) & .690 (.031) & .682 (.063) & .682 (.063) & .673 (.010) & .673 (.010)\\
{\sc HC(ANN)} w/ disc. & .876 (.018) & .876 (.018) & .670 (.046) & .670 (.046) & .738 (.028) & .738 (.028) & .688 (.017) & .688 (.017)\\
\bottomrule
\end{tabular}
\label{tab:AAAI,AUC,_1,preFS}
\end{table*}

\begin{table*}\centering
\caption{Our methods vs. the rest: mean classifier AUC for $\epsilon=0.2$, unbalanced datasets, pre-selected 4 features.}
\setlength{\tabcolsep}{2.7pt}\begin{tabular}{@{}lcccccccc@{}}\toprule
\multirow{2}{*}[-3pt]{\makecell[l]{Classifier}} & \multicolumn{2}{c}{Australia} & \multicolumn{2}{c}{Germany} & \multicolumn{2}{c}{Poland} & \multicolumn{2}{c}{Taiwan}\\
\cmidrule(l{7pt}r{7pt}){2-3} \cmidrule(l{7pt}r{7pt}){4-5} \cmidrule(l{7pt}r{7pt}){6-7} \cmidrule(l{7pt}r{7pt}){8-9}
& \multicolumn{1}{c}{Tru.} & \multicolumn{1}{c}{Str.} & \multicolumn{1}{c}{Tru.} & \multicolumn{1}{c}{Str.} & \multicolumn{1}{c}{Tru.} & \multicolumn{1}{c}{Str.} & \multicolumn{1}{c}{Tru.} & \multicolumn{1}{c}{Str.}\\ \midrule
{\sc HC(LR)} & .876 (.017) & .876 (.017) & .661 (.031) & .661 (.031) & .626 (.080) & .626 (.080) & .642 (.011) & .642 (.011)\\
{\sc HC(LR)} w/ disc. & .872 (.017) & .872 (.017) & .668 (.035) & .668 (.035) & .682 (.071) & .682 (.071) & .676 (.014) & .676 (.014)\\
{\sc HCapp(LR)} & .875 (.016) & .875 (.016) & .652 (.043) & .652 (.043) & .608 (.075) & .608 (.075) & .630 (.021) & .630 (.021)\\
{\sc HCapp(LR)} w/ disc. & .867 (.021) & .867 (.021) & .646 (.040) & .646 (.040) & .657 (.094) & .657 (.094) & .678 (.014) & .678 (.014)\\
{\sc HC(ANN)} & .866 (.021) & .866 (.021) & .654 (.037) & .654 (.037) & .674 (.053) & .674 (.053) & .654 (.007) & .654 (.007)\\
{\sc HC(ANN)} w/ disc. & .853 (.024) & .853 (.024) & .637 (.043) & .637 (.043) & .701 (.036) & .701 (.036) & .679 (.012) & .679 (.012)\\
\bottomrule
\end{tabular}
\label{tab:AAAI,AUC,_2,preFS}
\end{table*}

\begin{table*}\centering
\caption{Our methods vs. the rest: mean classifier AUC for $\epsilon=0.3$, unbalanced datasets, pre-selected 4 features.}
\setlength{\tabcolsep}{2.7pt}\begin{tabular}{@{}lcccccccc@{}}\toprule
\multirow{2}{*}[-3pt]{\makecell[l]{Classifier}} & \multicolumn{2}{c}{Australia} & \multicolumn{2}{c}{Germany} & \multicolumn{2}{c}{Poland} & \multicolumn{2}{c}{Taiwan}\\
\cmidrule(l{7pt}r{7pt}){2-3} \cmidrule(l{7pt}r{7pt}){4-5} \cmidrule(l{7pt}r{7pt}){6-7} \cmidrule(l{7pt}r{7pt}){8-9}
& \multicolumn{1}{c}{Tru.} & \multicolumn{1}{c}{Str.} & \multicolumn{1}{c}{Tru.} & \multicolumn{1}{c}{Str.} & \multicolumn{1}{c}{Tru.} & \multicolumn{1}{c}{Str.} & \multicolumn{1}{c}{Tru.} & \multicolumn{1}{c}{Str.}\\ \midrule
{\sc HC(LR)} & .855 (.019) & .855 (.019) & .648 (.034) & .648 (.034) & .634 (.067) & .634 (.067) & .641 (.008) & .641 (.008)\\
{\sc HC(LR)} w/ disc. & .849 (.020) & .849 (.020) & .656 (.032) & .656 (.032) & .649 (.076) & .649 (.076) & .669 (.012) & .669 (.012)\\
{\sc HCapp(LR)} & .854 (.016) & .854 (.016) & .644 (.040) & .644 (.040) & .602 (.093) & .602 (.093) & .628 (.023) & .628 (.023)\\
{\sc HCapp(LR)} w/ disc. & .846 (.019) & .846 (.019) & .646 (.043) & .646 (.043) & .621 (.094) & .621 (.094) & .669 (.012) & .669 (.012)\\
{\sc HC(ANN)} & .841 (.024) & .841 (.024) & .635 (.037) & .635 (.037) & .660 (.054) & .660 (.054) & .645 (.007) & .645 (.007)\\
{\sc HC(ANN)} w/ disc. & .830 (.026) & .830 (.026) & .622 (.039) & .622 (.039) & .666 (.037) & .666 (.037) & .668 (.012) & .668 (.012)\\
\bottomrule
\end{tabular}
\label{tab:AAAI,AUC,_3,preFS}
\end{table*}

\begin{table*}\centering
\caption{Our methods vs. the rest: mean classifier AUC for $\epsilon=0.4$, unbalanced datasets, pre-selected 4 features.}
\setlength{\tabcolsep}{2.7pt}\begin{tabular}{@{}lcccccccc@{}}\toprule
\multirow{2}{*}[-3pt]{\makecell[l]{Classifier}} & \multicolumn{2}{c}{Australia} & \multicolumn{2}{c}{Germany} & \multicolumn{2}{c}{Poland} & \multicolumn{2}{c}{Taiwan}\\
\cmidrule(l{7pt}r{7pt}){2-3} \cmidrule(l{7pt}r{7pt}){4-5} \cmidrule(l{7pt}r{7pt}){6-7} \cmidrule(l{7pt}r{7pt}){8-9}
& \multicolumn{1}{c}{Tru.} & \multicolumn{1}{c}{Str.} & \multicolumn{1}{c}{Tru.} & \multicolumn{1}{c}{Str.} & \multicolumn{1}{c}{Tru.} & \multicolumn{1}{c}{Str.} & \multicolumn{1}{c}{Tru.} & \multicolumn{1}{c}{Str.}\\ \midrule
{\sc HC(LR)} & .827 (.020) & .827 (.020) & .627 (.036) & .627 (.036) & .647 (.043) & .647 (.043) & .633 (.006) & .633 (.006)\\
{\sc HC(LR)} w/ disc. & .820 (.022) & .820 (.022) & .636 (.030) & .636 (.030) & .663 (.044) & .663 (.044) & .660 (.007) & .660 (.007)\\
{\sc HCapp(LR)} & .829 (.018) & .829 (.018) & .623 (.042) & .623 (.042) & .640 (.054) & .640 (.054) & .626 (.017) & .626 (.017)\\
{\sc HCapp(LR)} w/ disc. & .820 (.020) & .820 (.020) & .634 (.026) & .634 (.026) & .665 (.054) & .665 (.054) & .659 (.008) & .659 (.008)\\
{\sc HC(ANN)} & .809 (.026) & .809 (.026) & .612 (.038) & .612 (.038) & .658 (.032) & .658 (.032) & .637 (.006) & .637 (.006)\\
{\sc HC(ANN)} w/ disc. & .796 (.029) & .796 (.029) & .606 (.033) & .606 (.033) & .642 (.034) & .642 (.034) & .659 (.008) & .659 (.008)\\
\bottomrule
\end{tabular}
\label{tab:AAAI,AUC,_4,preFS}
\end{table*}

\begin{table*}\centering
\caption{Our methods vs. the rest: mean classifier AUC for $\epsilon=0.5$, unbalanced datasets, pre-selected 4 features.}
\setlength{\tabcolsep}{2.7pt}\begin{tabular}{@{}lcccccccc@{}}\toprule
\multirow{2}{*}[-3pt]{\makecell[l]{Classifier}} & \multicolumn{2}{c}{Australia} & \multicolumn{2}{c}{Germany} & \multicolumn{2}{c}{Poland} & \multicolumn{2}{c}{Taiwan}\\
\cmidrule(l{7pt}r{7pt}){2-3} \cmidrule(l{7pt}r{7pt}){4-5} \cmidrule(l{7pt}r{7pt}){6-7} \cmidrule(l{7pt}r{7pt}){8-9}
& \multicolumn{1}{c}{Tru.} & \multicolumn{1}{c}{Str.} & \multicolumn{1}{c}{Tru.} & \multicolumn{1}{c}{Str.} & \multicolumn{1}{c}{Tru.} & \multicolumn{1}{c}{Str.} & \multicolumn{1}{c}{Tru.} & \multicolumn{1}{c}{Str.}\\ \midrule
{\sc HC(LR)} & .788 (.024) & .788 (.024) & .611 (.033) & .611 (.033) & .645 (.033) & .645 (.033) & .622 (.008) & .622 (.008)\\
{\sc HC(LR)} w/ disc. & .779 (.024) & .779 (.024) & .616 (.027) & .616 (.027) & .655 (.030) & .655 (.030) & .646 (.006) & .646 (.006)\\
{\sc HCapp(LR)} & .788 (.023) & .788 (.023) & .608 (.039) & .608 (.039) & .644 (.035) & .644 (.035) & .618 (.013) & .618 (.013)\\
{\sc HCapp(LR)} w/ disc. & .780 (.024) & .780 (.024) & .617 (.033) & .617 (.033) & .658 (.026) & .658 (.026) & .644 (.006) & .644 (.006)\\
{\sc HC(ANN)} & .776 (.027) & .776 (.027) & .596 (.033) & .596 (.033) & .643 (.030) & .643 (.030) & .625 (.009) & .625 (.009)\\
{\sc HC(ANN)} w/ disc. & .768 (.028) & .768 (.028) & .583 (.032) & .583 (.032) & .623 (.031) & .623 (.031) & .645 (.007) & .645 (.007)\\
\bottomrule
\end{tabular}
\label{tab:AAAI,AUC,_5,preFS}
\end{table*}

\begin{table*}\centering
\caption{Our methods vs. the rest: mean classifier accuracy for $\epsilon=0.0$, balanced datasets, each classifier selects 4 best features.}
\setlength{\tabcolsep}{2.7pt}\begin{tabular}{@{}lcccccccc@{}}\toprule
\multirow{2}{*}[-3pt]{\makecell[l]{Classifier}} & \multicolumn{2}{c}{Australia} & \multicolumn{2}{c}{Germany} & \multicolumn{2}{c}{Poland} & \multicolumn{2}{c}{Taiwan}\\
\cmidrule(l{7pt}r{7pt}){2-3} \cmidrule(l{7pt}r{7pt}){4-5} \cmidrule(l{7pt}r{7pt}){6-7} \cmidrule(l{7pt}r{7pt}){8-9}
& \multicolumn{1}{c}{Tru.} & \multicolumn{1}{c}{Str.} & \multicolumn{1}{c}{Tru.} & \multicolumn{1}{c}{Str.} & \multicolumn{1}{c}{Tru.} & \multicolumn{1}{c}{Str.} & \multicolumn{1}{c}{Tru.} & \multicolumn{1}{c}{Str.}\\ \midrule
{\sc HC(LR)} & .859 (.015) & .859 (.015) & .671 (.030) & .671 (.030) & .726 (.028) & .726 (.028) & .686 (.006) & .686 (.006)\\
{\sc HC(LR)} w/ disc. & .855 (.017) & .855 (.017) & .674 (.027) & .674 (.027) & .741 (.031) & .741 (.031) & .695 (.005) & .695 (.005)\\
{\sc HCapp(LR)} & .860 (.015) & .860 (.015) & .673 (.032) & .673 (.032) & .709 (.032) & .709 (.032) & .647 (.018) & .647 (.018)\\
{\sc HCapp(LR)} w/ disc. & .851 (.028) & .851 (.028) & .653 (.051) & .653 (.051) & .729 (.034) & .729 (.034) & .634 (.006) & .634 (.006)\\
{\sc HC(ANN)} & .859 (.014) & .859 (.014) & .679 (.027) & .679 (.027) & .718 (.026) & .718 (.026) & .687 (.006) & .687 (.006)\\
{\sc HC(ANN)} w/ disc. & .847 (.016) & .847 (.016) & .676 (.025) & .676 (.025) & .739 (.030) & .739 (.030) & .696 (.005) & .696 (.005)\\
\bottomrule
\end{tabular}
\label{tab:AAAI,accuracy,_0_bal,FS}
\end{table*}

\begin{table*}\centering
\caption{Our methods vs. the rest: mean classifier accuracy for $\epsilon=0.1$, balanced datasets, each classifier selects 4 best features.}
\setlength{\tabcolsep}{2.7pt}\begin{tabular}{@{}lcccccccc@{}}\toprule
\multirow{2}{*}[-3pt]{\makecell[l]{Classifier}} & \multicolumn{2}{c}{Australia} & \multicolumn{2}{c}{Germany} & \multicolumn{2}{c}{Poland} & \multicolumn{2}{c}{Taiwan}\\
\cmidrule(l{7pt}r{7pt}){2-3} \cmidrule(l{7pt}r{7pt}){4-5} \cmidrule(l{7pt}r{7pt}){6-7} \cmidrule(l{7pt}r{7pt}){8-9}
& \multicolumn{1}{c}{Tru.} & \multicolumn{1}{c}{Str.} & \multicolumn{1}{c}{Tru.} & \multicolumn{1}{c}{Str.} & \multicolumn{1}{c}{Tru.} & \multicolumn{1}{c}{Str.} & \multicolumn{1}{c}{Tru.} & \multicolumn{1}{c}{Str.}\\ \midrule
{\sc HC(LR)} & .817 (.020) & .817 (.020) & .645 (.029) & .645 (.029) & .690 (.029) & .690 (.029) & .665 (.006) & .665 (.006)\\
{\sc HC(LR)} w/ disc. & .821 (.019) & .821 (.019) & .646 (.028) & .646 (.028) & .714 (.030) & .714 (.030) & .672 (.005) & .672 (.005)\\
{\sc HCapp(LR)} & .817 (.021) & .817 (.021) & .652 (.028) & .652 (.028) & .693 (.029) & .693 (.029) & .665 (.006) & .665 (.006)\\
{\sc HCapp(LR)} w/ disc. & .817 (.020) & .817 (.020) & .645 (.038) & .645 (.038) & .716 (.029) & .716 (.029) & .670 (.009) & .670 (.009)\\
{\sc HC(ANN)} & .804 (.024) & .804 (.024) & .633 (.032) & .633 (.032) & .672 (.032) & .672 (.032) & .665 (.006) & .665 (.006)\\
{\sc HC(ANN)} w/ disc. & .798 (.024) & .798 (.024) & .628 (.033) & .628 (.033) & .701 (.031) & .701 (.031) & .672 (.006) & .672 (.006)\\
\bottomrule
\end{tabular}
\label{tab:AAAI,accuracy,_1_bal,FS}
\end{table*}

\begin{table*}\centering
\caption{Our methods vs. the rest: mean classifier accuracy for $\epsilon=0.2$, balanced datasets, each classifier selects 4 best features.}
\setlength{\tabcolsep}{2.7pt}\begin{tabular}{@{}lcccccccc@{}}\toprule
\multirow{2}{*}[-3pt]{\makecell[l]{Classifier}} & \multicolumn{2}{c}{Australia} & \multicolumn{2}{c}{Germany} & \multicolumn{2}{c}{Poland} & \multicolumn{2}{c}{Taiwan}\\
\cmidrule(l{7pt}r{7pt}){2-3} \cmidrule(l{7pt}r{7pt}){4-5} \cmidrule(l{7pt}r{7pt}){6-7} \cmidrule(l{7pt}r{7pt}){8-9}
& \multicolumn{1}{c}{Tru.} & \multicolumn{1}{c}{Str.} & \multicolumn{1}{c}{Tru.} & \multicolumn{1}{c}{Str.} & \multicolumn{1}{c}{Tru.} & \multicolumn{1}{c}{Str.} & \multicolumn{1}{c}{Tru.} & \multicolumn{1}{c}{Str.}\\ \midrule
{\sc HC(LR)} & .794 (.023) & .794 (.023) & .624 (.032) & .624 (.032) & .677 (.031) & .677 (.031) & .648 (.005) & .648 (.005)\\
{\sc HC(LR)} w/ disc. & .795 (.022) & .795 (.022) & .631 (.031) & .631 (.031) & .693 (.032) & .693 (.032) & .649 (.005) & .649 (.005)\\
{\sc HCapp(LR)} & .793 (.019) & .793 (.019) & .631 (.032) & .631 (.032) & .680 (.030) & .680 (.030) & .648 (.005) & .648 (.005)\\
{\sc HCapp(LR)} w/ disc. & .794 (.021) & .794 (.021) & .629 (.042) & .629 (.042) & .692 (.032) & .692 (.032) & .648 (.005) & .648 (.005)\\
{\sc HC(ANN)} & .778 (.025) & .778 (.025) & .614 (.031) & .614 (.031) & .654 (.033) & .654 (.033) & .648 (.005) & .648 (.005)\\
{\sc HC(ANN)} w/ disc. & .771 (.025) & .771 (.025) & .608 (.036) & .608 (.036) & .679 (.036) & .679 (.036) & .647 (.005) & .647 (.005)\\
\bottomrule
\end{tabular}
\label{tab:AAAI,accuracy,_2_bal,FS}
\end{table*}

\begin{table*}\centering
\caption{Our methods vs. the rest: mean classifier accuracy for $\epsilon=0.3$, balanced datasets, each classifier selects 4 best features.}
\setlength{\tabcolsep}{2.7pt}\begin{tabular}{@{}lcccccccc@{}}\toprule
\multirow{2}{*}[-3pt]{\makecell[l]{Classifier}} & \multicolumn{2}{c}{Australia} & \multicolumn{2}{c}{Germany} & \multicolumn{2}{c}{Poland} & \multicolumn{2}{c}{Taiwan}\\
\cmidrule(l{7pt}r{7pt}){2-3} \cmidrule(l{7pt}r{7pt}){4-5} \cmidrule(l{7pt}r{7pt}){6-7} \cmidrule(l{7pt}r{7pt}){8-9}
& \multicolumn{1}{c}{Tru.} & \multicolumn{1}{c}{Str.} & \multicolumn{1}{c}{Tru.} & \multicolumn{1}{c}{Str.} & \multicolumn{1}{c}{Tru.} & \multicolumn{1}{c}{Str.} & \multicolumn{1}{c}{Tru.} & \multicolumn{1}{c}{Str.}\\ \midrule
{\sc HC(LR)} & .770 (.022) & .770 (.022) & .608 (.032) & .608 (.032) & .668 (.028) & .668 (.028) & .636 (.006) & .636 (.006)\\
{\sc HC(LR)} w/ disc. & .769 (.023) & .769 (.023) & .616 (.032) & .616 (.032) & .680 (.030) & .680 (.030) & .637 (.006) & .637 (.006)\\
{\sc HCapp(LR)} & .772 (.022) & .772 (.022) & .614 (.033) & .614 (.033) & .670 (.027) & .670 (.027) & .635 (.006) & .635 (.006)\\
{\sc HCapp(LR)} w/ disc. & .769 (.022) & .769 (.022) & .616 (.041) & .616 (.041) & .681 (.030) & .681 (.030) & .637 (.006) & .637 (.006)\\
{\sc HC(ANN)} & .756 (.028) & .756 (.028) & .599 (.034) & .599 (.034) & .644 (.032) & .644 (.032) & .635 (.006) & .635 (.006)\\
{\sc HC(ANN)} w/ disc. & .748 (.027) & .748 (.027) & .590 (.036) & .590 (.036) & .664 (.033) & .664 (.033) & .636 (.006) & .636 (.006)\\
\bottomrule
\end{tabular}
\label{tab:AAAI,accuracy,_3_bal,FS}
\end{table*}

\begin{table*}\centering
\caption{Our methods vs. the rest: mean classifier accuracy for $\epsilon=0.4$, balanced datasets, each classifier selects 4 best features.}
\setlength{\tabcolsep}{2.7pt}\begin{tabular}{@{}lcccccccc@{}}\toprule
\multirow{2}{*}[-3pt]{\makecell[l]{Classifier}} & \multicolumn{2}{c}{Australia} & \multicolumn{2}{c}{Germany} & \multicolumn{2}{c}{Poland} & \multicolumn{2}{c}{Taiwan}\\
\cmidrule(l{7pt}r{7pt}){2-3} \cmidrule(l{7pt}r{7pt}){4-5} \cmidrule(l{7pt}r{7pt}){6-7} \cmidrule(l{7pt}r{7pt}){8-9}
& \multicolumn{1}{c}{Tru.} & \multicolumn{1}{c}{Str.} & \multicolumn{1}{c}{Tru.} & \multicolumn{1}{c}{Str.} & \multicolumn{1}{c}{Tru.} & \multicolumn{1}{c}{Str.} & \multicolumn{1}{c}{Tru.} & \multicolumn{1}{c}{Str.}\\ \midrule
{\sc HC(LR)} & .751 (.024) & .751 (.024) & .598 (.034) & .598 (.034) & .656 (.030) & .656 (.030) & .627 (.006) & .627 (.006)\\
{\sc HC(LR)} w/ disc. & .747 (.026) & .747 (.026) & .604 (.033) & .604 (.033) & .667 (.030) & .667 (.030) & .627 (.006) & .627 (.006)\\
{\sc HCapp(LR)} & .753 (.024) & {\bf .753} (.024) & .602 (.033) & .602 (.033) & .656 (.030) & .656 (.030) & .627 (.006) & .627 (.006)\\
{\sc HCapp(LR)} w/ disc. & .747 (.026) & .747 (.026) & .601 (.039) & .601 (.039) & .670 (.029) & .670 (.029) & .627 (.006) & .627 (.006)\\
{\sc HC(ANN)} & .734 (.027) & .734 (.027) & .585 (.036) & .585 (.036) & .626 (.038) & .626 (.038) & .626 (.007) & .626 (.007)\\
{\sc HC(ANN)} w/ disc. & .727 (.029) & .727 (.029) & .578 (.035) & .578 (.035) & .654 (.035) & .654 (.035) & .626 (.006) & .626 (.006)\\
\bottomrule
\end{tabular}
\label{tab:AAAI,accuracy,_4_bal,FS}
\end{table*}

\begin{table*}\centering
\caption{Our methods vs. the rest: mean classifier accuracy for $\epsilon=0.5$, balanced datasets, each classifier selects 4 best features.}
\setlength{\tabcolsep}{2.7pt}\begin{tabular}{@{}lcccccccc@{}}\toprule
\multirow{2}{*}[-3pt]{\makecell[l]{Classifier}} & \multicolumn{2}{c}{Australia} & \multicolumn{2}{c}{Germany} & \multicolumn{2}{c}{Poland} & \multicolumn{2}{c}{Taiwan}\\
\cmidrule(l{7pt}r{7pt}){2-3} \cmidrule(l{7pt}r{7pt}){4-5} \cmidrule(l{7pt}r{7pt}){6-7} \cmidrule(l{7pt}r{7pt}){8-9}
& \multicolumn{1}{c}{Tru.} & \multicolumn{1}{c}{Str.} & \multicolumn{1}{c}{Tru.} & \multicolumn{1}{c}{Str.} & \multicolumn{1}{c}{Tru.} & \multicolumn{1}{c}{Str.} & \multicolumn{1}{c}{Tru.} & \multicolumn{1}{c}{Str.}\\ \midrule
{\sc HC(LR)} & .723 (.026) & .723 (.026) & .580 (.034) & .580 (.034) & .642 (.031) & .642 (.031) & .616 (.006) & .616 (.006)\\
{\sc HC(LR)} w/ disc. & .718 (.026) & .718 (.026) & .582 (.034) & .582 (.034) & .657 (.032) & .657 (.032) & .617 (.006) & .617 (.006)\\
{\sc HCapp(LR)} & .725 (.025) & .725 (.025) & .582 (.034) & .582 (.034) & .644 (.028) & .644 (.028) & .616 (.006) & .616 (.006)\\
{\sc HCapp(LR)} w/ disc. & .717 (.027) & .717 (.027) & .579 (.038) & .579 (.038) & .659 (.031) & {\bf .659} (.031) & .617 (.006) & .617 (.006)\\
{\sc HC(ANN)} & .704 (.030) & .704 (.030) & .567 (.034) & .567 (.034) & .618 (.033) & .618 (.033) & .616 (.006) & .616 (.006)\\
{\sc HC(ANN)} w/ disc. & .701 (.029) & .701 (.029) & .558 (.040) & .558 (.040) & .641 (.029) & .641 (.029) & .616 (.006) & .616 (.006)\\
\bottomrule
\end{tabular}
\label{tab:AAAI,accuracy,_5_bal,FS}
\end{table*}

\begin{table*}\centering
\caption{Our methods vs. the rest: mean classifier AUC for $\epsilon=0.0$, balanced datasets, each classifier selects 4 best features.}
\setlength{\tabcolsep}{2.7pt}\begin{tabular}{@{}lcccccccc@{}}\toprule
\multirow{2}{*}[-3pt]{\makecell[l]{Classifier}} & \multicolumn{2}{c}{Australia} & \multicolumn{2}{c}{Germany} & \multicolumn{2}{c}{Poland} & \multicolumn{2}{c}{Taiwan}\\
\cmidrule(l{7pt}r{7pt}){2-3} \cmidrule(l{7pt}r{7pt}){4-5} \cmidrule(l{7pt}r{7pt}){6-7} \cmidrule(l{7pt}r{7pt}){8-9}
& \multicolumn{1}{c}{Tru.} & \multicolumn{1}{c}{Str.} & \multicolumn{1}{c}{Tru.} & \multicolumn{1}{c}{Str.} & \multicolumn{1}{c}{Tru.} & \multicolumn{1}{c}{Str.} & \multicolumn{1}{c}{Tru.} & \multicolumn{1}{c}{Str.}\\ \midrule
{\sc HC(LR)} & .904 (.013) & .904 (.013) & .725 (.033) & .725 (.033) & .781 (.028) & .781 (.028) & .696 (.007) & .696 (.007)\\
{\sc HC(LR)} w/ disc. & .909 (.013) & {\bf .909} (.013) & .724 (.027) & .724 (.027) & .787 (.032) & .787 (.032) & .729 (.006) & .729 (.006)\\
{\sc HCapp(LR)} & .908 (.013) & .908 (.013) & .723 (.032) & .723 (.032) & .763 (.032) & .763 (.032) & .646 (.013) & .646 (.013)\\
{\sc HCapp(LR)} w/ disc. & .909 (.022) & {\bf .909} (.022) & .699 (.056) & .699 (.056) & .758 (.037) & .758 (.037) & .644 (.009) & .644 (.009)\\
{\sc HC(ANN)} & .906 (.013) & .906 (.013) & .737 (.028) & .737 (.028) & .778 (.028) & .778 (.028) & .700 (.009) & .700 (.009)\\
{\sc HC(ANN)} w/ disc. & .907 (.013) & .907 (.013) & .727 (.025) & .727 (.025) & .788 (.031) & .788 (.031) & .731 (.005) & .731 (.005)\\
\bottomrule
\end{tabular}
\label{tab:AAAI,AUC,_0_bal,FS}
\end{table*}

\begin{table*}\centering
\caption{Our methods vs. the rest: mean classifier AUC for $\epsilon=0.1$, balanced datasets, each classifier selects 4 best features.}
\setlength{\tabcolsep}{2.7pt}\begin{tabular}{@{}lcccccccc@{}}\toprule
\multirow{2}{*}[-3pt]{\makecell[l]{Classifier}} & \multicolumn{2}{c}{Australia} & \multicolumn{2}{c}{Germany} & \multicolumn{2}{c}{Poland} & \multicolumn{2}{c}{Taiwan}\\
\cmidrule(l{7pt}r{7pt}){2-3} \cmidrule(l{7pt}r{7pt}){4-5} \cmidrule(l{7pt}r{7pt}){6-7} \cmidrule(l{7pt}r{7pt}){8-9}
& \multicolumn{1}{c}{Tru.} & \multicolumn{1}{c}{Str.} & \multicolumn{1}{c}{Tru.} & \multicolumn{1}{c}{Str.} & \multicolumn{1}{c}{Tru.} & \multicolumn{1}{c}{Str.} & \multicolumn{1}{c}{Tru.} & \multicolumn{1}{c}{Str.}\\ \midrule
{\sc HC(LR)} & .892 (.015) & .892 (.015) & .685 (.033) & .685 (.033) & .750 (.030) & .750 (.030) & .675 (.007) & .675 (.007)\\
{\sc HC(LR)} w/ disc. & .890 (.017) & .890 (.017) & .686 (.028) & .686 (.028) & .764 (.031) & .764 (.031) & .705 (.006) & .705 (.006)\\
{\sc HCapp(LR)} & .892 (.015) & .892 (.015) & .694 (.032) & .694 (.032) & .748 (.029) & .748 (.029) & .675 (.006) & .675 (.006)\\
{\sc HCapp(LR)} w/ disc. & .890 (.017) & .890 (.017) & .685 (.041) & .685 (.041) & .758 (.030) & .758 (.030) & .704 (.005) & .704 (.005)\\
{\sc HC(ANN)} & .888 (.016) & .888 (.016) & .688 (.032) & .688 (.032) & .734 (.035) & .734 (.035) & .680 (.006) & .680 (.006)\\
{\sc HC(ANN)} w/ disc. & .877 (.019) & .877 (.019) & .669 (.033) & .669 (.033) & .757 (.033) & .757 (.033) & .706 (.006) & .706 (.006)\\
\bottomrule
\end{tabular}
\label{tab:AAAI,AUC,_1_bal,FS}
\end{table*}

\begin{table*}\centering
\caption{Our methods vs. the rest: mean classifier AUC for $\epsilon=0.2$, balanced datasets, each classifier selects 4 best features.}
\setlength{\tabcolsep}{2.7pt}\begin{tabular}{@{}lcccccccc@{}}\toprule
\multirow{2}{*}[-3pt]{\makecell[l]{Classifier}} & \multicolumn{2}{c}{Australia} & \multicolumn{2}{c}{Germany} & \multicolumn{2}{c}{Poland} & \multicolumn{2}{c}{Taiwan}\\
\cmidrule(l{7pt}r{7pt}){2-3} \cmidrule(l{7pt}r{7pt}){4-5} \cmidrule(l{7pt}r{7pt}){6-7} \cmidrule(l{7pt}r{7pt}){8-9}
& \multicolumn{1}{c}{Tru.} & \multicolumn{1}{c}{Str.} & \multicolumn{1}{c}{Tru.} & \multicolumn{1}{c}{Str.} & \multicolumn{1}{c}{Tru.} & \multicolumn{1}{c}{Str.} & \multicolumn{1}{c}{Tru.} & \multicolumn{1}{c}{Str.}\\ \midrule
{\sc HC(LR)} & .876 (.017) & .876 (.017) & .659 (.038) & .659 (.038) & .735 (.035) & .735 (.035) & .661 (.005) & .661 (.005)\\
{\sc HC(LR)} w/ disc. & .872 (.016) & .872 (.016) & .664 (.032) & .664 (.032) & .744 (.033) & .744 (.033) & .692 (.006) & .692 (.006)\\
{\sc HCapp(LR)} & .876 (.017) & .876 (.017) & .668 (.036) & .668 (.036) & .735 (.034) & .735 (.034) & .662 (.006) & .662 (.006)\\
{\sc HCapp(LR)} w/ disc. & .871 (.017) & .871 (.017) & .662 (.044) & .662 (.044) & .742 (.032) & .742 (.032) & .691 (.006) & .691 (.006)\\
{\sc HC(ANN)} & .868 (.021) & .868 (.021) & .657 (.033) & .657 (.033) & .718 (.036) & .718 (.036) & .666 (.007) & .666 (.007)\\
{\sc HC(ANN)} w/ disc. & .850 (.023) & .850 (.023) & .641 (.036) & .641 (.036) & .732 (.034) & .732 (.034) & .691 (.006) & .691 (.006)\\
\bottomrule
\end{tabular}
\label{tab:AAAI,AUC,_2_bal,FS}
\end{table*}

\begin{table*}\centering
\caption{Our methods vs. the rest: mean classifier AUC for $\epsilon=0.3$, balanced datasets, each classifier selects 4 best features.}
\setlength{\tabcolsep}{2.7pt}\begin{tabular}{@{}lcccccccc@{}}\toprule
\multirow{2}{*}[-3pt]{\makecell[l]{Classifier}} & \multicolumn{2}{c}{Australia} & \multicolumn{2}{c}{Germany} & \multicolumn{2}{c}{Poland} & \multicolumn{2}{c}{Taiwan}\\
\cmidrule(l{7pt}r{7pt}){2-3} \cmidrule(l{7pt}r{7pt}){4-5} \cmidrule(l{7pt}r{7pt}){6-7} \cmidrule(l{7pt}r{7pt}){8-9}
& \multicolumn{1}{c}{Tru.} & \multicolumn{1}{c}{Str.} & \multicolumn{1}{c}{Tru.} & \multicolumn{1}{c}{Str.} & \multicolumn{1}{c}{Tru.} & \multicolumn{1}{c}{Str.} & \multicolumn{1}{c}{Tru.} & \multicolumn{1}{c}{Str.}\\ \midrule
{\sc HC(LR)} & .853 (.022) & .853 (.022) & .641 (.038) & .641 (.038) & .726 (.028) & .726 (.028) & .653 (.006) & .653 (.006)\\
{\sc HC(LR)} w/ disc. & .846 (.024) & .846 (.024) & .646 (.036) & .646 (.036) & .728 (.031) & .728 (.031) & .675 (.007) & .675 (.007)\\
{\sc HCapp(LR)} & .853 (.022) & .853 (.022) & .647 (.038) & .647 (.038) & .727 (.027) & .727 (.027) & .653 (.006) & .653 (.006)\\
{\sc HCapp(LR)} w/ disc. & .846 (.023) & .846 (.023) & .645 (.042) & .645 (.042) & .728 (.031) & .728 (.031) & .677 (.006) & .677 (.006)\\
{\sc HC(ANN)} & .841 (.023) & .841 (.023) & .634 (.036) & .634 (.036) & .700 (.033) & .700 (.033) & .655 (.006) & .655 (.006)\\
{\sc HC(ANN)} w/ disc. & .825 (.028) & .825 (.028) & .615 (.036) & .615 (.036) & .718 (.034) & .718 (.034) & .677 (.006) & .677 (.006)\\
\bottomrule
\end{tabular}%
\label{tab:AAAI,AUC,_5_bal,FS}
\end{table*}

\begin{table*}\centering
\caption{Our methods vs. the rest: mean classifier accuracy for $\epsilon=0.0$, unbalanced datasets, each classifier selects 4 best features.}
\setlength{\tabcolsep}{2.7pt}%
\label{tab:AAAI,accuracy,_0,FS}
\end{table*}

\begin{table*}\centering
\caption{Our methods vs. the rest: mean classifier accuracy for $\epsilon=0.1$, unbalanced datasets, each classifier selects 4 best features.}
\setlength{\tabcolsep}{2.7pt}%
\label{tab:AAAI,accuracy,_1,FS}
\end{table*}

\begin{table*}\centering
\caption{Our methods vs. the rest: mean classifier accuracy for $\epsilon=0.2$, unbalanced datasets, each classifier selects 4 best features.}
\setlength{\tabcolsep}{2.7pt}%
\label{tab:AAAI,accuracy,_2,FS}
\end{table*}

\begin{table*}\centering
\caption{Our methods vs. the rest: mean classifier accuracy for $\epsilon=0.3$, unbalanced datasets, each classifier selects 4 best features.}
\setlength{\tabcolsep}{2.7pt}%
\label{tab:AAAI,accuracy,_5,FS}
\end{table*}

\begin{table*}\centering
\caption{Our methods vs. the rest: mean classifier AUC for $\epsilon=0.0$, unbalanced datasets, each classifier selects 4 best features.}
\setlength{\tabcolsep}{2.7pt}%
\label{tab:AAAI,AUC,_0,FS}
\end{table*}

\begin{table*}\centering
\caption{Our methods vs. the rest: mean classifier AUC for $\epsilon=0.1$, unbalanced datasets, each classifier selects 4 best features.}
\setlength{\tabcolsep}{2.7pt}%
\label{tab:AAAI,AUC,_1,FS}
\end{table*}

\begin{table*}\centering
\caption{Our methods vs. the rest: mean classifier AUC for $\epsilon=0.2$, unbalanced datasets, each classifier selects 4 best features.}
\setlength{\tabcolsep}{2.7pt}%
\label{tab:AAAI,AUC,_2,FS}
\end{table*}

\begin{table*}\centering
\caption{Our methods vs. the rest: mean classifier AUC for $\epsilon=0.3$, unbalanced datasets, each classifier selects 4 best features.}
\setlength{\tabcolsep}{2.7pt}%
\label{tab:AAAI,AUC,_5,FS}
\end{table*}

\begin{table*}\centering
\caption{Our methods vs. the rest: mean classifier accuracy for $\epsilon=0.0$, balanced datasets, all features.}
\setlength{\tabcolsep}{2.7pt}%
\label{tab:AAAI,accuracy,_0_bal,NFS}
\end{table*}

\begin{table*}\centering
\caption{Our methods vs. the rest: mean classifier accuracy for $\epsilon=0.1$, balanced datasets, all features.}
\setlength{\tabcolsep}{2.7pt}%
\label{tab:AAAI,accuracy,_1_bal,NFS}
\end{table*}

\begin{table*}\centering
\caption{Our methods vs. the rest: mean classifier accuracy for $\epsilon=0.2$, balanced datasets, all features.}
\setlength{\tabcolsep}{2.7pt}%
\label{tab:AAAI,accuracy,_2_bal,NFS}
\end{table*}

\begin{table*}\centering
\caption{Our methods vs. the rest: mean classifier accuracy for $\epsilon=0.3$, balanced datasets, all features.}
\setlength{\tabcolsep}{2.7pt}%
\label{tab:AAAI,accuracy,_3_bal,NFS}
\end{table*}

\clearpage
\begin{table*}\centering
\caption{Our methods vs. the rest: mean classifier accuracy for $\epsilon=0.4$, balanced datasets, all features.}
\setlength{\tabcolsep}{2.7pt}%
\label{tab:AAAI,accuracy,_5_bal,NFS}
\end{table*}

\begin{table*}\centering
\caption{Our methods vs. the rest: mean classifier AUC for $\epsilon=0.0$, balanced datasets, all features.}
\setlength{\tabcolsep}{2.7pt}%
\label{tab:AAAI,AUC,_0_bal,NFS}
\end{table*}

\begin{table*}\centering
\caption{Our methods vs. the rest: mean classifier AUC for $\epsilon=0.1$, balanced datasets, all features.}
\setlength{\tabcolsep}{2.7pt}%
\label{tab:AAAI,AUC,_1_bal,NFS}
\end{table*}

\begin{table*}\centering
\caption{Our methods vs. the rest: mean classifier AUC for $\epsilon=0.2$, balanced datasets, all features.}
\setlength{\tabcolsep}{2.7pt}%
\label{tab:AAAI,AUC,_2_bal,NFS}
\end{table*}

\begin{table*}\centering
\caption{Our methods vs. the rest: mean classifier AUC for $\epsilon=0.3$, balanced datasets, all features.}
\setlength{\tabcolsep}{2.7pt}%
\label{tab:AAAI,AUC,_5_bal,NFS}
\end{table*}

\begin{table*}\centering
\caption{Our methods vs. the rest: mean classifier accuracy for $\epsilon=0.0$, unbalanced datasets, all features.}
\setlength{\tabcolsep}{2.7pt}%
\label{tab:AAAI,accuracy,_0,NFS}
\end{table*}

\begin{table*}\centering
\caption{Our methods vs. the rest: mean classifier accuracy for $\epsilon=0.1$, unbalanced datasets, all features.}
\setlength{\tabcolsep}{2.7pt}%
\label{tab:AAAI,accuracy,_1,NFS}
\end{table*}

\begin{table*}\centering
\caption{Our methods vs. the rest: mean classifier accuracy for $\epsilon=0.2$, unbalanced datasets, all features.}
\setlength{\tabcolsep}{2.7pt}%
\label{tab:AAAI,accuracy,_2,NFS}
\end{table*}

\begin{table*}\centering
\caption{Our methods vs. the rest: mean classifier accuracy for $\epsilon=0.3$, unbalanced datasets, all features.}
\setlength{\tabcolsep}{2.7pt}%
\label{tab:AAAI,accuracy,_5,NFS}
\end{table*}

\begin{table*}\centering
\caption{Our methods vs. the rest: mean classifier AUC for $\epsilon=0.0$, unbalanced datasets, all features.}
\setlength{\tabcolsep}{2.7pt}%
\label{tab:AAAI,AUC,_0,NFS}
\end{table*}

\begin{table*}\centering
\caption{Our methods vs. the rest: mean classifier AUC for $\epsilon=0.1$, unbalanced datasets, all features.}
\setlength{\tabcolsep}{2.7pt}%
\label{tab:AAAI,AUC,_1,NFS}
\end{table*}

\begin{table*}\centering
\caption{Our methods vs. the rest: mean classifier AUC for $\epsilon=0.2$, unbalanced datasets, all features.}
\setlength{\tabcolsep}{2.7pt}%
\label{tab:AAAI,AUC,_2,NFS}
\end{table*}

\begin{table*}\centering
\caption{Our methods vs. the rest: mean classifier AUC for $\epsilon=0.3$, unbalanced datasets, all features.}
\setlength{\tabcolsep}{2.7pt}%
\label{tab:AAAI,AUC,_5,NFS}
\end{table*}

\end{document}